%% file: iclr2027_conference.tex
\documentclass{article}
\usepackage{iclr2027_conference,times}
\input{math_commands.tex}

\usepackage{hyperref}
\usepackage{url}
\usepackage{bm, mathtools, amsmath, amssymb, amsthm, cancel, booktabs, multirow, tikz, subcaption, wrapfig, float, placeins}
\usepackage[table]{xcolor}

\newtheorem{lemma}{Lemma}
\newtheorem{proposition}{Proposition}
\newtheorem{theorem}[proposition]{Theorem}
\newtheorem{observation}[proposition]{Observation}
\usetikzlibrary{positioning,arrows.meta,calc}

\title{Probability-Flow Distillation: \\ Distribution Matching in Parameter Space}

\author{Rohith Ramanan \\
Department of Physics \\
Wadhwani School of Data Science \& AI \\
Indian Institute of Technology Madras \\
\texttt{rrohithiam@gmail.com} \\
\And
A. N. Rajagopalan \\
Department of Electrical Engineering \\
Indian Institute of Technology Madras \\
\texttt{raju@ee.iitm.ac.in} \\
}

\iclrfinalcopy
\begin{document}

\maketitle

\input{sections/00_abstract}

\input{sections/01_introduction}

\input{sections/02_preliminaries}

\input{sections/03_sds_and_sdi}

\input{sections/04_pfd}

\input{sections/05_experiments}

\input{sections/06_closing}

\bibliography{iclr2027_conference}
\bibliographystyle{iclr2027_conference}

\clearpage

\appendix

\input{sections/A_proof_prop_sds}

\input{sections/B_proof_prop_sdi}

\input{sections/C_proof_prop_pfd}

\input{sections/D_proof_thm_db}

\input{sections/E_proof_thm_pfd}

\clearpage

\input{sections/F_ddim_inv}

\clearpage

\input{sections/G_grouping}

\clearpage

\input{sections/H_implementation}

\clearpage

\input{sections/I_qualitative}

\end{document}

%% file: math_commands.tex
\usepackage{amsmath,amsfonts,bm}

\def\eqref#1{equation~\ref{#1}}
\def\1{\bm{1}}

\DeclareMathAlphabet{\mathsfit}{\encodingdefault}{\sfdefault}{m}{sl}
\SetMathAlphabet{\mathsfit}{bold}{\encodingdefault}{\sfdefault}{bx}{n}

%% file: sections/00_abstract.tex
\begin{abstract}
Score distillation methods use pretrained diffusion models as priors for optimizing parameters through differentiable forward models, most notably in text-to-3D generation. Yet the distribution they induce over those parameters is not well understood. Observing that existing distillation methods reduce to one of three—Score Distillation Sampling (SDS), Score Distillation via Inversion (SDI), and Variational Score Distillation (VSD)—we extend the particle variational inference view of VSD to the other two. We show that SDS collapses onto the modes of the target, while SDI converges to a contracted version of it, and explain why SDI needs a negative classifier-free guidance scale. Next, we observe that the DDIM posterior mean equals a single Euler step of the probability-flow ODE (PF-ODE). Replacing this step in SDI with a full reverse solve makes the target a fixed point, but it requires solving two concatenated PF-ODEs. Dropping a Jacobian from the resulting gradient gives Probability-Flow Distillation (PFD), which requires solving only the forward PF-ODE. Experiments on synthetic targets, the CelebA dataset, and text-to-3D generation support our analysis and demonstrate the practical effectiveness of PFD.
\end{abstract}

%% file: sections/01_introduction.tex
\section{Introduction}
\label{sec:introduction}

Diffusion and related generative models~\citep{ncsn,ddpm,ldm,imagen,flow} have demonstrated remarkable success in sampling from complex data distributions. However, in some applications the quantity of interest is not the data itself but the parameters that generate the data through a differentiable system, i.e.,
\begin{equation}
\bm{x}_0 = \bm{g}(\bm{\theta},\bm{\omega}), \qquad \bm{\omega}\sim\xi,
\label{eqn:1}
\end{equation}
where $\bm{\theta}\in\Theta$, $\bm{x}_0\in\mathcal{X}$, and $\bm{g}:\Theta\times\Omega\rightarrow\mathcal{X}$ is differentiable in $\bm{\theta}$. Here $\bm{\omega}$ is the randomness of the forward model itself. The target in the parameter space $\Theta$ is a distribution $\pi$ such that the pushforward of the coupling $(\pi,\xi)$ under $\bm{g}$ is the data distribution, i.e., $\bm{g}_\#(\pi\otimes\xi) = p_0$, so that $\bm{g}(\bm{\theta},\bm{\omega})$ is distributed as the data. While modern generative models provide an effective way to sample in $\mathcal{X}$, they offer no direct method for sampling in $\Theta$. This has motivated a family of distillation\footnote{This differs from the common use of the term ``distillation''---methods that accelerate sampling in $\mathcal{X}$.} algorithms that use a pretrained generative model to sample in the parameter space.

\citet{sds} introduced this idea with Score Distillation Sampling (SDS) for text-to-3D generation, where $\bm\theta$ denotes the parameters of a 3D representation~\citep{nerf,instantngp,dmtet,3dgs}, $\bm{g}$ a differentiable renderer, and $\bm{\omega}$ the camera pose it renders from. The motivation was the scarcity of high-quality 3D data. While models trained directly on 3D assets enable faster generation and exhibit greater multi-view consistency, existing 3D datasets consist largely of synthetic assets and capture far less of the diversity and detail of real-world appearances than large-scale image datasets. Therefore, distilling 2D diffusion priors for 3D generation remains an active research direction. These methods, however, address a broader class of inverse and parameter-optimization problems. For example, \citet{audiosds} use a pretrained audio diffusion model to optimize the parameters of an impact-synthesizer and also to separate sound sources from a text prompt. On the other hand, in Distribution Matching Distillation~\citep{dmd}, $\bm\theta$ are the weights of a one-step generator and $\bm\omega\sim\xi=\mathcal{N}(\bm{0},\bm{I})$ is its input, so that $\bm{g}(\bm\theta,\bm\omega) = G_{\bm\theta}(\bm\omega)$ and the target is $\pi = \delta_{\bm\theta^\star}$, such that $(G_{\bm\theta^\star})_\#\xi = p_0$.

We observe that subsequent text-to-3D distillation methods can be viewed---either directly or as approximations---as instances of one of these three methods (see Appendix~\ref{app:grouping}): (i) SDS, (ii) Score Distillation via Inversion (SDI)~\citep{sdi}, and (iii) Variational Score Distillation (VSD)~\citep{vsd}. They can all be written abstractly as follows:
\begin{equation}
\frac{\mathrm{d}\bm\theta}{\mathrm{d}\tau} = -\mathbb{E}_{t,\bm\omega,\bm\epsilon} \left[\bm\Delta(\bm{x}_0,t,\bm\epsilon) \, \frac{\partial \bm{x}_0}{\partial\bm\theta} \right],
\label{eqn:2}
\end{equation}
where $\bm\Delta$ is the gradient estimator used by the method and $\bm\epsilon$ is the Gaussian noise it draws. If the gradient estimator is deterministic given $\bm{x}_0$, we drop the $\bm\epsilon$ argument and write $\bm\Delta(\bm{x}_0,t)$. The Jacobian $\partial\bm{x}_0/\partial\bm\theta$, taken at the sampled $\bm{\omega}$, transfers the update from the data space to the parameter space. Here $\tau$ is the optimization time of the distillation algorithm, distinct from the diffusion time $t$; at each step, a new $t\sim\mathcal{U}[0,1]$ is sampled while $\tau$ continues to evolve. Thus, the optimization is not tied to the finite horizon of the diffusion model's schedule. We view \eqref{eqn:2} as dynamics on a set of particles $\{\bm{\theta}_i\}_i$ and write $q_{0,\tau} = \bm{g}_\#(\pi_\tau\otimes\xi)$ for the induced distribution of $\bm{x}_0$ at time $\tau$. Only VSD was explicitly derived in this particle variational inference (ParVI) formulation, with $q_{0,\tau}\to p_0$ as its goal. Its gradient estimator is given by:
\begin{equation}
\bm{\Delta}(\bm{x}_0,t,\bm\epsilon) = w(t)\left(\nabla_{\bm{x}}\log q_{t,\tau}(\bm{x}_t) -\nabla_{\bm{x}}\log p_t(\bm{x}_t)\right),\qquad\bm{x}_t=\bm{x}_0+\sigma_t\bm{\epsilon},
\label{eqn:3}
\end{equation}
under which $q_{0,\tau}$ follows the Wasserstein gradient flow of $\mathcal{F}[q]=\mathbb{E}_t[w(t)\,D_{\mathrm{KL}}(q_{t}[q]\,\|\,p_t)]$. Here $q_t = q * \mathcal{N}(\bm{0},\sigma_t^2\bm{I}_d)$, and the same relates $p_t$ to $p_0$. The other two are not derived in this way, so the question of which distribution results does not arise in their formulations.

In this work, we extend the study to SDS and SDI and characterize the distribution to which $q_{0,\tau}$ converges under each. In particular, we show that SDS collapses onto the modes of $p_0$, and SDI converges to a contracted version of it. This explains the qualitative differences between the 3D assets these methods produce. The ParVI view of SDI also explains the need for the negative classifier-free guidance (CFG) scale that the method requires. We then make the observation that the DDIM posterior mean can be viewed as a single Euler step of the probability-flow ODE (PF-ODE), which suggests a distribution matching extension of SDI. However, this requires solving two concatenated ODEs. We address this by introducing Probability-Flow Distillation (PFD), a simpler variant of the method which, surprisingly, converges faster than the extension itself. In our experiments, it is also on par with or better than VSD. We evaluate PFD against SDS, SDI, and VSD in settings of increasing complexity in both $p_0$ and $\bm{g}$. We begin with sampling from three synthetic targets: a one-dimensional Gaussian, a two-dimensional funnel, and a twenty-dimensional mixture of Gaussians (MoG). We then use a diffusion model trained on the CelebA dataset for four tasks: image generation, symmetric image generation, inpainting, and vector graphics generation. Finally, we turn to the natural testbed of these methods, text-to-3D generation. Across these settings, the results agree with the analysis and show that PFD performs well in practice.

%% file: sections/02_preliminaries.tex
\section{Preliminaries}
\label{sec:preliminaries}

\subsection{Probability-Flow ODE}
\label{subsec:pf-ode}

Consider the Itô SDE, $\mathrm{d}\bm{x}_t = \bm{f}(\bm{x}_t,t)\,\mathrm{d}t + g(t)\,\mathrm{d}\bm{w}_t$, where $t \in [0,1]$, $\bm{f}$ is the drift, $g$ is the diffusion coefficient, and $\bm{w}_t$ is the standard Brownian motion. With $p_t(\bm{x}_t)$ denoting the marginal density of $\bm{x}_t \in \mathbb{R}^d$, \citet{anderson} derived the corresponding reverse-time SDE:
\begin{equation}
\mathrm{d}\bm{x}_t = \left[\bm{f}(\bm{x}_t,t) - g(t)^2 \nabla_{\bm{x}} \log p_t(\bm{x}_t)\right]\mathrm{d}t + g(t)\,\mathrm{d}\bar{\bm{w}}_t,
\label{eqn:4}
\end{equation}
where $\bar{\bm{w}}_t$ is the reverse-time Brownian motion. The quantity $\nabla_{\bm{x}} \log p_t(\bm{x}_t)$, known as the score function, is generally intractable but can be approximated via denoising score matching (DSM)~\citep{vincent} by minimizing the following objective:
\begin{equation}
\mathcal{L}_{\mathrm{DSM}}(\phi) = \mathbb{E}_{\bm{x}_0\sim p_0,\bm{\epsilon}\sim\mathcal{N}(\bm{0},\bm{I}_d),t} \left[ w(t)\left\| \bm{\epsilon} - \bm{\epsilon}_\phi(\bm{x}_t,t) \right\|^2 \right],\quad \bm{x}_t \sim p(\cdot\,|\,\bm{x}_0), \quad w(t)> 0.
\label{eqn:5}
\end{equation}
The score function is then recovered up to a time-dependent scaling factor, i.e., $\nabla_{\bm{x}} \log p_t(\bm{x}_t) \propto -\bm{\epsilon}_\phi(\bm{x}_t,t)$ (assuming $\bm{\epsilon}_\phi$ is the exact minimizer of \eqref{eqn:5}). Associated with the forward SDE is an ODE that preserves the same marginal densities $p_t(\bm{x}_t)$ for all $t$, known as the probability-flow ODE~\citep{song}:
\begin{equation}
\frac{\mathrm{d}\bm{x}_t}{\mathrm{d}t} = \bm{f}(\bm{x}_t,t) - \frac{1}{2} g(t)^2 \nabla_{\bm{x}} \log p_t(\bm{x}_t).
\label{eqn:6}
\end{equation}

\subsection{Diffusion models}
\label{subsec:diff_models}

Diffusion models can be written as SDEs of the form given in Section~\ref{subsec:pf-ode}. Two choices of drift and diffusion coefficients are standard, both drifts being linear in $\bm{x}_t$, so that the perturbation kernel $p(\bm{x}_t\,|\,\bm{x}_0)$ is Gaussian and available in closed form. This is what makes training tractable: a sample at any time $t$ can be drawn directly from $\bm{x}_0$ without simulating the SDE. For the Variance-Exploding (VE) SDE, $\mathrm{d}\bm{x}_t = \sqrt{\mathrm{d}\sigma_t^2/\mathrm{d}t}\,\mathrm{d}\bm{w}_t$, the perturbation kernel admits the reparameterization $\bm{x}_t = \bm{x}_0 + \sigma_t\bm{\epsilon}$, and the score is recovered from the trained network as $\nabla_{\bm{x}}\log p_t(\bm{x}_t) = -\bm{\epsilon}_\phi(\bm{x}_t,t)/\sigma_t$. For the Variance-Preserving (VP) SDE, $\mathrm{d}\bm{x}_t = -\tfrac{1}{2}\beta_t\bm{x}_t\,\mathrm{d}t + \sqrt{\beta_t}\,\mathrm{d}\bm{w}_t$, the corresponding reparameterization is $\bm{x}_t = \sqrt{\alpha_t}\,\bm{x}_0 + \sqrt{1-\alpha_t}\,\bm{\epsilon}$ with $\alpha_t = \exp(-\int_0^t\beta(s)\,\mathrm{d}s)$, and $\nabla_{\bm{x}}\log p_t(\bm{x}_t) = -\bm{\epsilon}_\phi(\bm{x}_t,t)/\sqrt{1-\alpha_t}$. For the VE-SDE, the PF-ODE is given by $\mathrm{d}\bm{x}_t/\mathrm{d}t = \bm{\epsilon}_\phi(\bm{x}_t,t)\,\mathrm{d}\sigma_t/\mathrm{d}t$. A change of variables brings the VP case to the same form: with $\tilde\sigma_t \coloneqq \sqrt{(1-\alpha_t)/\alpha_t}$ and $\tilde{\bm{x}}_t \coloneqq \bm{x}_t/\sqrt{\alpha_t}$, the VP reparameterization becomes $\tilde{\bm{x}}_t = \bm{x}_0 + \tilde\sigma_t\bm{\epsilon}$, with the PF-ODE given by:
\begin{equation}
\frac{\mathrm{d}\tilde{\bm{x}}_t}{\mathrm{d}t} = \bm{\epsilon}_\phi(\bm{x}_t,t)\,\frac{\mathrm{d}\tilde\sigma_t}{\mathrm{d}t}, \qquad \bm{x}_t = \frac{\tilde{\bm{x}}_t}{\sqrt{1+\tilde\sigma_t^2}},
\label{eqn:7}
\end{equation}
which is exactly the DDIM-ODE~\citep{ddim}. We therefore work in the VE parametrization throughout. Given the deterministic nature of the DDIM-ODE, it can be integrated from $0\to t$, mapping $\bm{x}_0$ uniquely to a noisy state $\bm{x}_t$. This procedure is referred to as DDIM inversion.

\subsection{Classifier-free guidance}
\label{subsec:cfg}

Classifier-free guidance~\citep{cfg} is the de facto technique for controlling the conditioning strength. During training, the conditioning signal $c$ is randomly dropped so that the network learns both a conditional predictor $\bm{\epsilon}_\phi(\bm{x}_t,t,c)$ and an unconditional predictor $\bm{\epsilon}_\phi(\bm{x}_t,t,\varnothing)$. At inference, the two are linearly extrapolated to form the guided noise estimate,
\begin{equation}
\hat{\bm{\epsilon}}_\phi(\bm{x}_t,t,c) = \bm{\epsilon}_\phi(\bm{x}_t,t,\varnothing) + \gamma\big(\bm{\epsilon}_\phi(\bm{x}_t,t,c) - \bm{\epsilon}_\phi(\bm{x}_t,t,\varnothing)\big).
\label{eqn:8}
\end{equation}
Since each predictor approximates the score of its respective marginal, i.e., $\bm{\epsilon}_\phi(\bm{x}_t,t,\cdot) = -\sigma_t\nabla_{\bm{x}}\log p_t(\bm{x}_t\,|\,\cdot)$, we have $\hat{\bm{\epsilon}}_\phi(\bm{x}_t,t,c) = -\sigma_t\nabla_{\bm{x}}\log \tilde{p}_t(\bm{x}_t\,|\,c)$, where
\begin{equation}
\tilde{p}_t(\bm{x}_t\,|\,c) \;\propto\; p_t(\bm{x}_t\,|\,c)\left(\frac{p_t(\bm{x}_t\,|\,c)}{p_t(\bm{x}_t\,|\,\varnothing)}\right)^{\gamma-1}.
\label{eqn:9}
\end{equation}
Therefore, CFG replaces the conditional distribution $p_t(\bm{x}_t\,|\,c)$ with a tilted version of it, reweighted by the likelihood ratio between the conditional and the unconditional marginals; sampling is driven away from $p_t(\cdot\,|\,\varnothing)$ and towards $p_t(\cdot\,|\,c)$. Setting $\gamma=1$ removes the tilt and recovers the conditional model, while $\gamma>1$ amplifies alignment with $c$, typically at the expense of diversity.

%% file: sections/03_sds_and_sdi.tex
\section{Analysis of SDS and SDI}
\label{sec:sds_sdi}

\subsection{SDS collapses onto the modes}
\label{subsec:sds}

SDS starts from the idea that $\bm{x}_0 = \bm{g}(\bm{\theta},\bm{\omega})$ should be likely under the pretrained model for every draw of $\bm{\omega}$. Since the likelihood is intractable, the training loss in \eqref{eqn:5} is used as a surrogate, and the parameters are obtained as $\bm{\theta}^\star = \arg\min_{\bm{\theta}} \mathbb{E}_{\bm{\omega}}\big[\mathcal{L}_{\mathrm{DSM}}(\phi, \bm{x}_0 = \bm{g}(\bm{\theta},\bm{\omega}))\big]$, now minimized over $\bm{\theta}$ with $\phi$ held fixed. Differentiating with respect to $\bm{\theta}$ gives:
\begin{equation}
\nabla_{\bm{\theta}}\mathcal{L}_{\mathrm{DSM}} \;\propto\; \mathbb{E}_{t,\bm{\omega},\bm{\epsilon}} \left[w(t) \left(\bm{\epsilon}_\phi(\bm{x}_t,t) - \bm{\epsilon}\right) \frac{\partial\bm{\epsilon}_\phi(\bm{x}_t,t)}{\partial\bm{x}_t} \frac{\partial\bm{x}_0}{\partial\bm{\theta}} \right].
\label{eqn:10}
\end{equation}
The middle factor is the Jacobian of the score network, a scaled Hessian of the log-density, which is expensive to evaluate and ill-conditioned at small noise levels. SDS omits it and uses $\bm{\Delta}=w(t)\big(\bm{\epsilon}_\phi(\bm{x}_t,t) - \bm{\epsilon}\big)$ as the gradient estimator; taking $w(t) = \sigma_t$, it can be written as $\bm{\Delta} = \bm{x}_0-\hat{\bm{x}}_0(\bm{x}_t)$, where
\begin{equation}
\hat{\bm{x}}_0(\bm{x}_t) = \bm{x}_t - \sigma_t\bm{\epsilon}_\phi(\bm{x}_t,t) = \bm{x}_t + \sigma_t^2\nabla_{\bm{x}}\log p_t(\bm{x}_t) = \mathbb{E}[\bm{x}_0\,|\,\bm{x}_t]
\label{eqn:11}
\end{equation}
is the posterior mean of $\bm{x}_0$ given the noisy observation $\bm{x}_t$. In this section and the next, we take the forward model to be the identity map, i.e., $\bm{x}_0 = \bm{g}(\bm{\theta},\bm{\omega}) = \bm{\theta}$, so that the parameters themselves are the data and $\bm{g}$ carries no randomness. This simplifies the exposition and the notation. The general case follows by the chain rule and an expectation over $\bm{\omega}$, and we will return to it in the experiments.

We now extend SDS to a set of particles $\{\bm{x}_{0,i}\}_i$. At every step, a diffusion time $t$ and independent noises $\{\bm{\epsilon}_i\}_i$ are sampled. Each particle $\bm{x}_{0,i}$ is then  perturbed to $\bm{x}_{t,i}=\bm{x}_{0,i}+\sigma_{t}\bm{\epsilon}_i$, from which $\bm{\Delta}_i=\sigma_{t}(\bm{\epsilon}_\phi(\bm{x}_{t,i},t)-\bm{\epsilon}_i)$ is computed, and the particles are updated. The target is not a fixed point of these dynamics. To see this, suppose $q_{0,\tau}=p_0$ and assume $p_0$ has finite second moments and $\mathrm{Cov}(\bm{x}_0)\succ 0$. Applying the law of total covariance to the coupling $(\bm{x}_0,\bm{x}_t)$, and using $\hat{\bm{x}}_0=\mathbb{E}[\bm{x}_0\,|\,\bm{x}_t]$, gives:
\begin{equation}
\mathrm{Cov}(\hat{\bm{x}}_0) = \mathrm{Cov}(\bm{x}_0) - \mathbb{E}\!\left[\mathrm{Cov}(\bm{x}_0\,|\,\bm{x}_t)\right].
\label{eqn:12}
\end{equation}
The subtracted term is positive definite for $\sigma_t>0$: if $\bm{v}^{\!\top}\mathbb{E}[\mathrm{Cov}(\bm{x}_0\,|\,\bm{x}_t)]\bm{v}=0$ for some $\bm{v}\neq\bm{0}$, then $\mathrm{Var}(\bm{v}^{\!\top}\bm{x}_0\,|\,\bm{x}_t)=0$ a.s., so $\bm{v}^{\!\top}\bm{x}_0=h(\bm{x}_t)$ for some measurable $h$, i.e., there is some non-zero direction $\bm{v}$ for which $\bm{v}^{\!\top}\bm{x}_0$ is completely determined by the noisy observation $\bm{x}_t$. Since $\bm{x}_t=\bm{x}_0+\sigma_t\bm{\epsilon}$ has full support given $\bm{x}_0$, this forces $h$ to be constant and $\mathrm{Var}(\bm{v}^{\!\top}\bm{x}_0)=0$, contradicting $\mathrm{Cov}(\bm{x}_0)\succ 0$, hence $\mathrm{Cov}(\bm{x}_0)\succ\mathrm{Cov}(\hat{\bm{x}}_0)$. Also, note that $\mathbb{E}[\hat{\bm{x}}_0]=\mathbb{E}[\bm{x}_0]$ by the tower property. At $q_{0,\tau}=p_0$, the update therefore returns a distribution with the same mean but a strictly smaller covariance, so $p_0$ is not a fixed point.

The particles instead converge onto the modes of $p_0$, which can be seen by writing out the dynamics. Substituting $\bm{\epsilon}_\phi(\bm{x}_t,t)=-\sigma_t\nabla_{\bm{x}}\log p_t(\bm{x}_t)$ gives $\bm{\Delta}=-\sigma_t^2\nabla_{\bm{x}}\log p_t(\bm{x}_t)-\sigma_t\bm{\epsilon}$. Since $\mathbb{E}_{\bm{\epsilon}}[\bm{\epsilon}]=\bm{0}$, the second term vanishes under expectation, so each particle evolves as
\begin{equation}
\frac{\mathrm{d}\bm{x}_0}{\mathrm{d}\tau}=-\mathbb{E}_{t,\bm{\epsilon}}\left[\bm{\Delta}(\bm{x}_0,t,\bm{\epsilon})\right] =\mathbb{E}_{t,\bm{\epsilon}}\left[\sigma_t^2\,\nabla_{\bm{x}}\log p_t(\bm{x}_0+\sigma_t\bm{\epsilon})\right].
\label{eqn:13}
\end{equation}
Only the target score appears on the right, and no term in the update accounts for $q_{0,\tau}$. In contrast, we show in Section~\ref{subsec:sdi} that SDI does account for the current distribution---if only through an approximation. Therefore, SDS, viewed as a ParVI method, doesn't have a term that maintains the entropy of the ensemble, and each particle independently converges to a local maximum of $p_0$, with particles starting in different basins converging to different maxima. The following proposition considers the case of isotropic Gaussians. We choose this family because it also simplifies the corresponding analysis for the other methods in the sections that follow.

\begin{proposition}
Let $p_0 = \mathcal{N}(\bm{\mu},\rho^2\bm{I}_d)$ be the target and take the VE schedule $\sigma_t = t$. Then the SDS update reduces to the following linear ODE:
\begin{subequations}
\begin{align}
\frac{\mathrm{d}\bm{x}_0}{\mathrm{d}\tau} &= \mathbb{E}_{t,\bm{\epsilon}} \left[ t^2\,\nabla_{\bm{x}}\log p_t(\bm{x}_0+t\bm{\epsilon}) \right] = -A(\rho)\,(\bm{x}_0-\bm{\mu}),
\label{eqn:14a}
\\[2pt]
A(\rho) &\coloneqq \mathbb{E}_t\!\left[\frac{t^2}{\rho^2+t^2}\right]  = 1-\rho\arctan\frac{1}{\rho} \;>\;0 \quad\forall\rho.
\label{eqn:14b}
\end{align}
\end{subequations}
Since \eqref{eqn:14a} is affine in $\bm{x}_0$, an isotropic Gaussian initialization $q_{0,0}=\mathcal{N}(\bm{m}_0,\nu_0^2\bm{I}_d)$ remains in that family for all $\tau$. Writing $q_{0,\tau} = \mathcal{N}(\bm{m}_\tau,\nu_\tau^2\bm{I}_d)$, we get $\bm{m}_\tau = \bm{\mu} + e^{-A(\rho)\tau}(\bm{m}_0-\bm{\mu})$ and $\nu_\tau^2 = \nu_0^2\,e^{-2A(\rho)\tau}$. Both decay exponentially, so $q_{0,\tau}$ converges to the degenerate Gaussian $\delta_{\bm{\mu}}$.
\label{prop:sds}
\end{proposition}

\subsection{SDI converges to a contracted target}
\label{subsec:sdi}

SDI builds on the observation that SDS resembles a high-variance version of DDIM and replaces the randomly sampled noise $\bm{\epsilon}$ with the noise implied by DDIM inversion. The gradient estimator retains the form $\bm{\Delta}(\bm{x}_0,t) = \bm{x}_0 - \hat{\bm{x}}_0(\bm{x}_t)$, with $\bm{x}_t$ now obtained from $\bm{x}_0$ by DDIM inversion.

\paragraph{DDIM inversion as approximate forward PF-ODE.}
Instead of sampling $\bm{x}_t$ from $p(\cdot\,|\,\bm{x}_0)$, if we map $\bm{x}_0$ to $\bm{x}_t$ through the PF-ODE induced by $\{q_{t,\tau}\}_t$, solved in the forward direction---writing $\bm{\Phi}^{q_\tau}_{0\to t}$ for its flow map, so that $\bm{x}_t = \bm{\Phi}^{q_\tau}_{0\to t}(\bm{x}_0)$---the state it produces has the same marginal distribution as before, but the map is now driven by $\nabla_{\bm{x}}\log q_{t,\tau}$ and thus explicitly depends on $q_{0,\tau}$. DDIM inversion can be thought of as an approximation of $\bm{\Phi}^{q_\tau}_{0\to t}$, integrating the score of the pretrained model under the null prompt in place of that of the particle distribution, $\nabla_{\bm{x}}\log q_{t,\tau}(\bm{x}_t) \approx \nabla_{\bm{x}}\log p_t(\bm{x}_t\,|\,\varnothing)$. Thus, it can be viewed as providing an approximate reference to the ensemble. It can be argued that this substitution is reasonable in the setting for which it was designed: the $\bm{g}(\bm\theta_i, \bm\omega)$ are renderings of a 3D asset, and throughout optimization, they remain close to the manifold of natural images.

\paragraph{Explanation of negative CFG scale in SDI.}
The approximation above describes the inversion without guidance. However, in practice, the inversion is run with CFG, and we now ask what the approximation implies in that case. In image editing, DDIM inversion with a positive CFG scale under a source prompt describing the input, maps the image to a latent, and DDIM sampling, again with a positive scale but under the target prompt, gives the edited image. In the text-to-3D setting, no prompt describes the current renderings, so SDI inverts under the target prompt itself. Combined with the approximation above, if we associate the unconditioned marginal with the particle distribution and the conditional one with the target, i.e., $p_t(\cdot\,|\,\varnothing)\approx q_{t,\tau}$ and $p_t(\cdot\,|\, c) = p_t$, the guided score of \eqref{eqn:9} becomes
\begin{equation}
\nabla_{\bm{x}}\log \tilde{p}_t(\bm{x}_t\,|\,c)\;\approx\;\nabla_{\bm{x}}\log\!\left[\,p_t(\bm{x}_t)\left(\frac{p_t(\bm{x}_t)}{q_{t,\tau}(\bm{x}_t)}\right)^{\gamma-1}\right],
\label{eqn:15}
\end{equation}
so CFG drives the PF-ODE induced by $\{p_t\}_t$, away from the source $q_{t,\tau}$ and towards the target $p_t$. The inversion, however, integrates the PF-ODE induced by $\{q_{t,\tau}\}_t$, so its guidance must tilt the other way, exchanging $p_t$ and $q_{t,\tau}$ in \eqref{eqn:15}. We therefore need $\bm{\epsilon}_\phi(\bm{x}_t,t,c) + \gamma(\bm{\epsilon}_\phi(\bm{x}_t,t,\varnothing)-\bm{\epsilon}_\phi(\bm{x}_t,t,c))$ in place of the standard $\bm{\epsilon}_\phi(\bm{x}_t,t,\varnothing) + \gamma(\bm{\epsilon}_\phi(\bm{x}_t,t,c)-\bm{\epsilon}_\phi(\bm{x}_t,t,\varnothing))$ of \eqref{eqn:8}. Rewriting it in the form the implementation expects gives
\begin{equation}
\begin{split}
&\bm{\epsilon}_\phi(\bm{x}_t,t,c)+\gamma\big(\bm{\epsilon}_\phi(\bm{x}_t,t,\varnothing)-\bm{\epsilon}_\phi(\bm{x}_t,t,c)\big) \\
&\qquad = \bm{\epsilon}_\phi(\bm{x}_t,t,\varnothing)+(1-\gamma)\big(\bm{\epsilon}_\phi(\bm{x}_t,t,c)-\bm{\epsilon}_\phi(\bm{x}_t,t,\varnothing)\big),
\end{split}
\label{eqn:16}
\end{equation}
so realizing an effective guidance scale $\gamma$ in inversion corresponds to a user-specified CFG scale $\gamma_{\mathrm{cfg}} = 1-\gamma$. For $\gamma>1$, this gives $\gamma_{\mathrm{cfg}}<0$, explaining the negative CFG scale needed in SDI. Note that $\gamma_{\mathrm{cfg}} = 0$ recovers the unguided inversion of the previous paragraph, the approximation of $\bm{\Phi}^{q_\tau}_{0\to t}$ with no CFG. A positive scale would tilt it the wrong way and Appendix~\ref{app:ddim_editing} shows that this makes editing unstable, while the negative scale restores stability.

Under the ideal version of SDI---in which $\bm{x}_t$ is obtained through the forward PF-ODE rather than its DDIM inversion approximation---the particle dynamics are given by:
\begin{equation}
\frac{\mathrm{d}\bm{x}_0}{\mathrm{d}\tau} = -\mathbb{E}_t\left[\bm{\Delta}(\bm{x}_0,t)\right] = \mathbb{E}_t\!\left[\hat{\bm{x}}_0\!\left(\bm{\Phi}^{q_\tau}_{0\to t}(\bm{x}_0)\right)\right] - \bm{x}_0 .
\label{eqn:17}
\end{equation}
The optimization converges when $\mathbb{E}_t[\hat{\bm{x}}_0(\bm{\Phi}^{q_\tau}_{0\to t}(\bm{x}_0))] = \bm{x}_0$ holds $q_{0,\tau}$-almost surely. Unlike SDS, the right-hand side of \eqref{eqn:17} depends on $q_{0,\tau}$, and this dependence acts as an entropy term that resists collapse. By a similar argument to that of Section~\ref{subsec:sds}, however, the target is not a fixed point of SDI either, since the posterior mean contracts the spread here as well. The limiting distribution, therefore, does not collapse, but its spread is strictly smaller than that of $p_0$. The following proposition makes this explicit for isotropic Gaussian $p_0$.

\begin{proposition}
With $p_0 = \mathcal{N}(\bm{\mu},\rho^2\bm{I}_d)$ and the VE schedule $\sigma_t = t$, both $\bm{\Phi}^{q_\tau}_{0\to t}$ and $\hat{\bm{x}}_0$ are affine, so an isotropic Gaussian initialization remains in that family for all $\tau$, and the mean and the standard deviation follow
\begin{subequations}
\begin{align}
\frac{\mathrm{d}\bm{m}_\tau}{\mathrm{d}\tau} &= -A(\rho)\,(\bm{m}_\tau-\bm{\mu}),
\label{eqn:18a}
\\[2pt]
\frac{\mathrm{d}\nu_\tau}{\mathrm{d}\tau} &= \left(B(\nu_\tau)-1\right)\nu_\tau, \qquad B(\nu) \coloneqq \frac{\rho^2}{\nu}\, \mathbb{E}_t\!\left[\frac{\sqrt{\nu^2+t^2}}{\rho^2+t^2}\right],
\label{eqn:18b}
\end{align}
\end{subequations}
with $A(\rho)$ as in Proposition~\ref{prop:sds} and \eqref{eqn:18b} has a unique stationary point $\nu_\star$, given by $B(\nu_\star)=1$, which is globally attracting and satisfies $\nu_\star<\rho$. The particles, therefore, converge to $\mathcal{N}(\bm{\mu},\nu_\star^2\bm{I}_d)$, which shares the mean of $p_0$ but has a strictly smaller variance.
\label{prop:sdi}
\end{proposition}

\begin{figure}
\centering
\includegraphics[width=\linewidth]{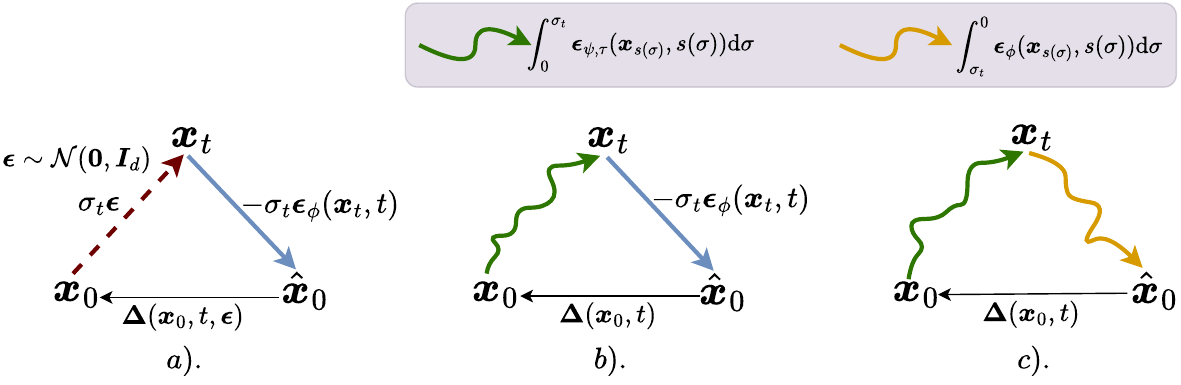}
\caption{(a) SDS: $\bm{x}_t$ is obtained by Gaussian perturbation of $\bm{x}_0$, and $\hat{\bm{x}}_0$ by the posterior mean. (b) SDI: $\bm{x}_t$ is obtained by solving the forward PF-ODE, and $\hat{\bm{x}}_0$ still by the posterior mean. (c) Both are obtained by integrating the PF-ODE, removing the posterior-mean approximation. The green and orange curved arrows denote integration of the forward and reverse PF-ODEs, driven by $\bm{\epsilon}_{\psi,\tau} = -\sigma_t\nabla_{\bm{x}}\log q_{t,\tau}$ and $\bm{\epsilon}_\phi = -\sigma_t\nabla_{\bm{x}}\log p_t$ respectively.}
\label{fig:grad}
\end{figure}

\subsection{Posterior mean as a single Euler step}
\label{subsec:dual_bridge}

\begin{observation}
In terms of the noise level $\sigma$, the VE PF-ODE takes the form $\mathrm{d}\bm{x}/\mathrm{d}\sigma = \bm{\epsilon}_\phi(\bm{x}, t(\sigma))$. A single explicit Euler step of this ODE from $\sigma_t$ to $0$ gives $\bm{x}_t - \sigma_t\bm{\epsilon}_\phi(\bm{x}_t,t) = \hat{\bm{x}}_0(\bm{x}_t)$. The posterior mean is therefore exactly one Euler step of the reverse PF-ODE, from noise to data.
\end{observation}

The observation above suggests a dual-bridge extension to SDI, illustrated in Figure~\ref{fig:grad}, in which we solve the reverse PF-ODE rather than taking the posterior mean. Denoting its flow map as $\bm{\Phi}^{p}_{t\to0}$, the new gradient estimator becomes
\begin{equation}
\bm{\Delta}(\bm{x}_0,t)
=
\bm{x}_0 - \bm{\Phi}^{p}_{t\to0}\!\left(\bm{\Phi}^{q_\tau}_{0\to t}(\bm{x}_0)\right),
\label{eqn:19}
\end{equation}
for which the target is now a fixed point. When $q_{0,\tau} = p_0$, $\bm{\Phi}^{p}_{t\to0}$ inverts $\bm{\Phi}^{q_\tau}_{0\to t}$, and $\bm{\Delta}$ vanishes for every particle.

\begin{theorem}
For the estimator in \eqref{eqn:19}, the following holds:
\begin{subequations}
\begin{align}
\bm{\Delta}(\bm{x}_0,t) &= \int_0^t \frac{1}{2}g_s^2\, \frac{\partial\bm{\Phi}^{p}_{s\to0}}{\partial\bm{x}}(\bm{x}_s) \left( \nabla_{\bm{x}}\log q_{s,\tau}(\bm{x}_s) - \nabla_{\bm{x}}\log p_s(\bm{x}_s) \right)\mathrm{d}s,
\label{eqn:20a}
\\[2pt]
\mathbb{E}_t[\bm{\Delta}(\bm{x}_0,t)] &= \int_0^1 \frac{1}{2}(1-s)\,g_s^2\, \frac{\partial\bm{\Phi}^{p}_{s\to0}}{\partial\bm{x}}(\bm{x}_s) \left( \nabla_{\bm{x}}\log q_{s,\tau}(\bm{x}_s) - \nabla_{\bm{x}}\log p_s(\bm{x}_s) \right)\mathrm{d}s,
\label{eqn:20b}
\end{align}
\end{subequations}
where $\bm{x}_s = \bm{\Phi}^{q_\tau}_{0\to s}(\bm{x}_0)$,  $g_s$ is the diffusion coefficient of Section~\ref{subsec:pf-ode}, and $\partial\bm{\Phi}^{p}_{s\to0}/\partial\bm{x}$ is the Jacobian of the reverse PF-ODE started at time $s$.
\label{thm:db}
\end{theorem}

Taking $t=1$ gives $\bm{\Delta}(\bm{x}_0,1)$, which is exactly the Dual Diffusion Implicit Bridge (DDIB)~\citep{ddib} between the particle distribution and the target. The expected gradient in \eqref{eqn:20b} computes the same integral but with a different weight, so each update moves the particles an infinitesimal step in a direction correlated with DDIB, explaining the tendency of $q_{0,\tau}$ towards $p_0$.

%% file: sections/04_pfd.tex
\section{Probability-Flow Distillation}
\label{sec:pfd}

The dual-bridge method solves two concatenated PF-ODEs and therefore doubles the NFEs per step relative to SDI, which is already expensive compared to the one or two evaluations of SDS and VSD. Looking at \eqref{eqn:20b}, were it not for the Jacobian $\partial\bm{\Phi}^{p}_{t\to0}/\partial\bm{x}$, the gradient could be evaluated from the forward PF-ODE alone. We therefore define Probability-Flow Distillation (PFD), under which the particles evolve as:
\begin{equation}
\frac{\mathrm{d}\bm{x}_0}{\mathrm{d}\tau} = \int_0^1 w(t) \left( \nabla_{\bm{x}}\log p_t(\bm{x}_t) - \nabla_{\bm{x}}\log q_{t,\tau}(\bm{x}_t) \right)\mathrm{d}t,
\label{eqn:21}
\end{equation}
where $\bm{x}_t = \bm{\Phi}^{q_\tau}_{0\to t}(\bm{x}_0)$ and
$w(t) = \tfrac{1}{2}(1-t)\,g_t^2$; although this choice of $w(t)$ is not restrictive, the results below hold for any $w(t)>0$. Like the other methods, PFD could be implemented by sampling $t$ at each step and using the gradient estimator in \eqref{eqn:20a} without the Jacobian. However, \eqref{eqn:21} already gives the expression for the exact PFD gradient, so we can evaluate it directly, with the asymptotic cost per step unchanged. This avoids the Monte Carlo variance from sampling $t$, which can reduce the number of optimization steps needed for convergence in practice.

To keep the comparison fair, we replace the expectation over $t\sim\mathcal{U}[0,1]$ in the other baselines with an integral over $[0,1]$, evaluated on $N$ quadrature nodes $\{t_i\}_{i=1}^N$: SDS integrates $\bm{x}_0-\hat{\bm{x}}_0(\bm{x}_{t_i})$ with $\bm{x}_{t_i}$ given by a fresh Gaussian perturbation of $\bm{x}_0$ at each node, SDI integrates the same but with $\{\bm{x}_{t_i}\}_i$ obtained from a single solve of the forward PF-ODE, and VSD integrates the score difference in \eqref{eqn:21} with $\bm{x}_{t_i}$ obtained as in SDS. SDS requires one evaluation of the pretrained model $\bm{\epsilon}_\phi$ per node, while SDI, VSD, and PFD require two: one of $\bm{\epsilon}_\phi$ and one of $\bm{\epsilon}_\psi$, an auxiliary model trained online to approximate the score $\nabla_{\bm{x}}\log q_{s,\tau}$; in all four cases, the cost grows as $\mathcal{O}(N)$. In contrast, implementing the dual-bridge method in the same way requires solving a reverse PF-ODE at every node $t_i$, and thus needs $\mathcal{O}(N^2)$ NFEs per step. We therefore include it only in the synthetic experiments, where it serves as a check that removing the posterior-mean approximation does lead to distribution-matching.

\paragraph{Comparison with VSD.}
(i). VSD has the same gradient form as PFD, but with an expectation over $\bm\epsilon$. The two differ only in the coupling $(\bm{x}_0, \bm{x}_t)$, which is stochastic in the former case and deterministic in the latter. Therefore, a Monte Carlo variance remains in realizing the VSD gradient, whereas the PFD gradient is deterministic, up to $\bm{\omega}$. (ii). VSD is the Wasserstein gradient of $\mathbb{E}_t[w(t)\,D_{\mathrm{KL}}(q_{t,\tau}\,\|\,p_t)]$ and therefore gives the steepest descent on that functional. The following theorem shows that when one looks instead at $D_{\mathrm{KL}}(q_{0,\tau}\,\|\,p_0)$, PFD decreases it faster than VSD when a small time range $[0,\eta]$ is considered.

\begin{theorem}
Under the VE schedule $\sigma_t = t$ restricted to $[0,\eta]$, the particles under VSD and PFD evolve as follows:
\begin{subequations}
\begin{align}
\left.\frac{\mathrm{d}\bm{x}_0}{\mathrm{d}\tau}\right|_{\mathrm{VSD}} &= \int_0^\eta w(t)\,\mathbb{E}_{\bm{\epsilon}}\!\left[ \nabla_{\bm{x}}\log p_t(\bm{x}_t) -\nabla_{\bm{x}}\log q_{t,\tau}(\bm{x}_t) \right]\mathrm{d}t, \qquad \bm{x}_t = \bm{x}_0+t\bm{\epsilon},
\label{eqn:22a}\\[2pt]
\left.\frac{\mathrm{d}\bm{x}_0}{\mathrm{d}\tau}\right|_{\mathrm{PFD}} &= \int_0^\eta w(t)\left( \nabla_{\bm{x}}\log p_t(\bm{x}_t) -\nabla_{\bm{x}}\log q_{t,\tau}(\bm{x}_t) \right)\mathrm{d}t, \qquad \bm{x}_t = \bm{\Phi}^{q_\tau}_{0\to t}(\bm{x}_0).
\label{eqn:22b}
\end{align}
\end{subequations}
Let $\mathrm{d}D_{\mathrm{KL}}/\mathrm{d}\tau$ denote the rate of change of $D_{\mathrm{KL}}(q_{0,\tau}\,\|\,p_0)$. Then, under an update from a common $q_{0,\tau}$, assuming $q_{0,\tau}$ decays at infinity, we have:
\begin{equation}
\left.\frac{\mathrm{d}D_{\mathrm{KL}}}{\mathrm{d}\tau}\right|_{\mathrm{VSD}} - \left.\frac{\mathrm{d}D_{\mathrm{KL}}}{\mathrm{d}\tau}\right|_{\mathrm{PFD}} = C_\eta \left( \mathbb{E}_{\bm{x}_0\sim q_{0,\tau}} \!\left[ \left\| \nabla^2\log p_0(\bm{x}_0) - \nabla^2\log q_{0,\tau}(\bm{x}_0) \right\|_F^2 \right] + O(\eta^2)\right).
\label{eqn:23}
\end{equation}
Here $C_\eta = \tfrac12\int_0^\eta w(t)\,t^2\,\mathrm{d}t$, which is strictly positive for any $w(t)>0$, and the leading coefficient is non-negative. PFD therefore decreases $D_{\mathrm{KL}}(q_{0,\tau}\,\|\,p_0)$ at least as fast as VSD for all sufficiently small $\eta$.
\label{thm:pfd}
\end{theorem}

Even though this theorem is for small $\eta$, PFD remains competitive with VSD for $\eta$ close to $1$. The following proposition makes the case for isotropic Gaussians, where PFD converges faster than both VSD and the dual-bridge (DB) method. 

\begin{proposition}
Set $\eta=1$ in \eqref{eqn:22a} and \eqref{eqn:22b} for VSD and PFD; DB follows the latter with the extra Jacobian factor $\partial\bm{\Phi}^{p}_{t\to0}/\partial\bm{x}$. All three keep the isotropic Gaussian family invariant and admit $p_0$ as their unique, globally attracting fixed point. Writing
\begin{equation}
\left.\frac{\mathrm{d}\bm{m}_\tau}{\mathrm{d}\tau}\right|_{\mathrm{M}} = -A_{\mathrm{M}}\,(\bm{m}_\tau-\bm{\mu}), \qquad \left.\frac{\mathrm{d}\nu_\tau}{\mathrm{d}\tau}\right|_{\mathrm{M}} = (\rho^2-\nu_\tau^2)\,B_{\mathrm{M}}(\nu_\tau),
\label{eqn:24}
\end{equation}
for $\mathrm{M}\in\{\mathrm{PFD},\mathrm{VSD},\mathrm{DB}\}$, we have $A_{\mathrm{PFD}} = A_{\mathrm{VSD}} > A_{\mathrm{DB}}$ and $B_{\mathrm{PFD}} > B_{\mathrm{VSD}}, B_{\mathrm{DB}}$, with no ordering between $B_{\mathrm{VSD}}$ and $B_{\mathrm{DB}}$. Therefore, PFD decreases $D_{\mathrm{KL}}(q_{0,\tau}\,\|\,p_0)$ strictly faster than DB, and strictly faster than VSD whenever $\nu_\tau \neq \rho$.
\label{prop:pfd}
\end{proposition}

%% file: sections/05_experiments.tex
\section{Experiments}
\label{sec:experiments}

\subsection{Sampling from Synthetic Targets}
\label{subsec:synthetic_targets}

We consider three targets of increasing difficulty: (i)~a one-dimensional Gaussian $p_0 = \mathcal{N}(2, 0.5^2)$; (ii)~Neal's two-dimensional funnel, $x_1 \sim \mathcal{N}(0, 1.5^2)$ and $x_2 \mid x_1 \sim \mathcal{N}(0, e^{x_1})$; and (iii)~a MoG with $10$ components in $\mathbb{R}^{20}$, $p_0(\bm{x}) = \sum_{k=1}^{10} \pi_k\, \mathcal{N}(\bm{x};\, \bm{\mu}_k,\, v_k \bm{I}_{20})$, whose parameters are drawn once as $\bm{\pi} \sim \mathrm{Dir}(8\cdot\bm{1}_{10})$, $\bm{\mu}_k \sim \mathcal{N}(\bm{0}, 3^2\bm{I}_{20})$, and $v_k \sim \mathcal{U}[0.35, 0.45]$, and then fixed. For the Gaussian and the mixture, the scores $\nabla_{\bm{x}}\log p_t$ are available in closed form. For the funnel, we use a score network trained using DSM on $10^5$ samples from $p_0$.

All integrals over $t$ are evaluated with $8$ Gauss--Legendre nodes, and every run takes $1500$ Euler steps in $\tau$. For the Gaussian target, every method keeps the Gaussian family invariant, so the score $\nabla_{\bm{x}}\log q_{t,\tau}$ and both flow maps are available in closed form, evaluated at the empirical mean and standard deviation of the particles. For the funnel and the mixture, $\bm{\epsilon}_{\psi,\tau}$ is a MLP trained online with DSM on the particles, with $1000$ warm-up Adam steps followed by $4$ steps per update. The forward PF-ODE is solved with Euler steps on the quadrature nodes, and DB solves each reverse PF-ODE with $7$ Euler steps. Figure~\ref{fig:1d_gauss} shows the results for the Gaussian target, where the propositions of Sections~\ref{sec:sds_sdi} and~\ref{sec:pfd} can be checked directly, and Figure~\ref{fig:nd_toy} shows the results for the funnel and the MoG. In all cases, SDS collapses onto the modes of $p_0$, SDI converges to a contracted version of $p_0$, and DB, VSD, and PFD match $p_0$. Figure~\ref{fig:1d_gauss}(a) shows that $D_{\mathrm{KL}}(q_{0,\tau}\,\|\,p_0)$ keeps decreasing for DB, VSD, and PFD, fastest for PFD, followed by VSD and then DB, whereas it plateaus for SDS and SDI. Figure~\ref{fig:1d_gauss}(c) compares the deterministic gradient of PFD with the noisy gradient of VSD.

\begin{figure}[ht]
\centering
\includegraphics[width=\linewidth]{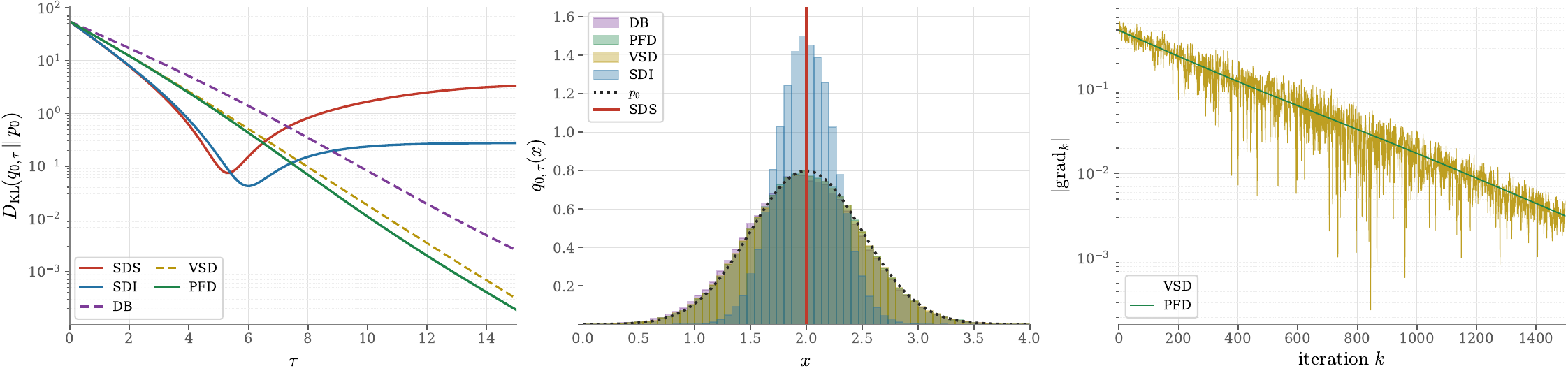}
\caption{One-dimensional Gaussian target $\mathcal{N}(2, 0.5^2)$ with $10^5$ particles. (left) $D_{\mathrm{KL}}(q_{0,\tau}\,\|\,p_0)$ against optimization time. (middle) Final particle distributions against $p_0$, the vertical line marks the point mass reached by SDS. (right) Gradient magnitude of VSD and PFD for a single particle across iterations.}
\label{fig:1d_gauss}
\end{figure}

\begin{figure}[ht]
\centering
\includegraphics[width=\linewidth]{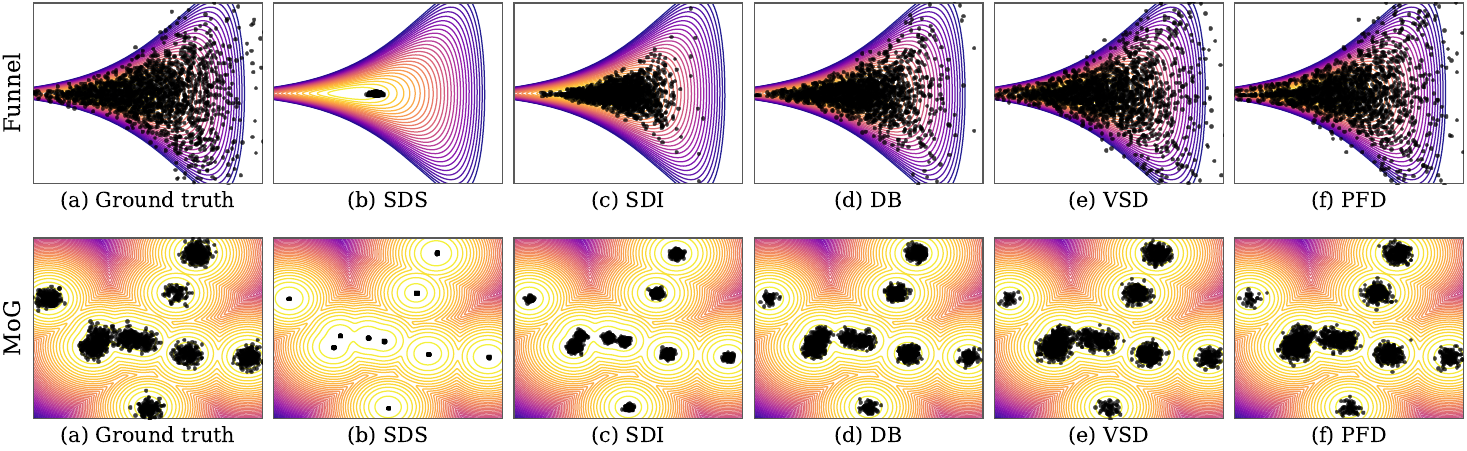}
\caption{$2000$ particles after optimization on the two-dimensional funnel (top) and the twenty-dimensional MoG (bottom), shown against contours of the target density. The mixture is projected onto the two leading principal components of the ground-truth samples, shown in panel (a) for reference.}
\label{fig:nd_toy}
\end{figure}

\subsection{Experiments on the CelebA dataset}
\label{subsec:celeba}

We train a VE diffusion model on the CelebA dataset~\citep{celeba} and use it as the pretrained model for four forward models (all deterministic, so $\bm\omega$ plays no role): (i)~image generation, $\bm{g}(\bm{\theta}) = \bm{\theta}$; (ii)~symmetric image generation, $\bm{g}(\bm{\theta}) = [\bm{\theta}, \operatorname{flip}(\bm{\theta})]$, which mirrors half of an image about the vertical axis; (iii)~image inpainting, $\bm{g}(\bm{\theta}) = \bm{M} \odot \bm{\theta} + (\bm{1} - \bm{M}) \odot \bm{x}^{\mathrm{obs}}$, where $\bm{M}$ is a fixed binary mask and $\bm{x}^{\mathrm{obs}}$ is the observed image; and (iv)~vector graphics optimization, where $\bm{g}$ is the differentiable rasterizer of \texttt{diffvg}~\citep{diffvg} and $\bm{\theta}$ contains the control points, colors, and widths of a set of B\'ezier strokes. For image generation, we run each method with four different seeds and $2.5$k particles per seed, giving $10$k samples per method. Table~\ref{tab:celeba_fid} reports their FID, computed with \texttt{clean-fid}~\citep{cleanfid}, and CLIP-MMD (CMMD)~\citep{cmmd}. As references, we also report both metrics for the Euler and Heun samplers and for a second, disjoint $10$k subset of the dataset. Both metrics give the same ordering: SDS is furthest from the data distribution, followed by SDI, while VSD and PFD perform similarly. Qualitative results for all four forward models are given in Appendices~\hyperref[app:qualitative]{I.1--I.4}, and details of the pretrained model, its training, and the setup of each experiment are given in Appendix~\ref{app:celeb_implementation}.

\begin{table}[ht]
\centering
\caption{FID and CMMD against a fixed reference set of $10$k dataset images. Dataset scores a second, disjoint $10$k subset against it, and Euler 50 and Heun 50 are 50-step PF-ODE samplers.}
\label{tab:celeba_fid}
\begin{tabular}{l ccc cccc}
\toprule
& \multicolumn{3}{c}{Reference} & \multicolumn{4}{c}{Distillation} \\
\cmidrule(lr){2-4} \cmidrule(lr){5-8}
& Dataset & Euler 50 & Heun 50 & SDS & SDI & VSD & \cellcolor{gray!15}PFD \\
\midrule
FID $\downarrow$  & $2.89$  & $25.54$ & $24.43$ & $236.82$ & $108.78$ & $\bm{16.61}$ & \cellcolor{gray!15}$16.92$ \\
CMMD $\downarrow$ & $<0.001$ & $0.377$ & $0.368$ & $6.531$  & $3.289$  & $\bm{0.106}$ & \cellcolor{gray!15}$\bm{0.106}$ \\
\bottomrule
\end{tabular}
\end{table}

\subsection{Text-to-3D Generation}
\label{subsec:textto3d}

This is the first setting in which $\bm{\omega}$ is non-degenerate, as the camera pose is drawn afresh at every step so that $\bm{g}(\bm{\theta},\bm{\omega})$ renders the 3D representation $\bm{\theta}$ from that view. We build on the \texttt{threestudio}~\citep{threestudio} implementations of SDS, SDI, and VSD, modified so that all methods have the same per-step cost, and implement PFD in the same framework. Every method is run for $5$k steps; the implementation details are given in Appendix~\ref{app:text-to-3d-implementation}. Table~\ref{tab:t23d} reports CLIP R-precision~\citep{clipr} under three backbones, which measures how well the renders match their prompt, and ImageReward~\citep{imagereward}, which is trained to predict human preferences. Both are averaged over $120$ rendered views of the assets generated from $22$ prompts. Qualitative comparisons are given in Appendix~\ref{app:text_to_3d}. Together with the CelebA results, they show what the analysis of Section~\ref{sec:sds_sdi} looks like in practice. On CelebA, SDS gives nearly identical samples, and SDI provides varied but blurry ones, while VSD and PFD yield varied and detailed ones. The same ordering in detail is visible in the 3D assets, where SDS produces over-saturated, washed-out surfaces and SDI sharper but still smooth ones, while VSD and PFD recover fine geometry and texture.

\begin{table}[ht]
\centering
\caption{Quantitative results for text-to-3D generation on $22$ prompts, reporting CLIP R-precision under three backbones and ImageReward, averaged over $120$ rendered views per asset.}
\label{tab:t23d}
\setlength{\tabcolsep}{9pt}
\begin{tabular}{cl cccc}
\toprule
& & SDS & SDI & VSD & \cellcolor{gray!15}PFD \\
\midrule
\multirow{3}{*}{CLIP R-precision $\uparrow$}
  & ViT-L/14 & $0.796$ & $0.925$ & $0.958$      & \cellcolor{gray!15}$\bm{0.969}$ \\
  & ViT-B/16 & $0.820$ & $0.917$ & $\bm{0.955}$ & \cellcolor{gray!15}$\bm{0.955}$ \\
  & ViT-B/32 & $0.802$ & $0.886$ & $\bm{0.998}$ & \cellcolor{gray!15}$\bm{0.998}$ \\
\midrule
ImageReward $\uparrow$ & & $-1.150$ & $-0.257$ & $0.680$ & \cellcolor{gray!15}$\bm{0.891}$ \\
\bottomrule
\end{tabular}
\end{table}

%% file: sections/06_closing.tex
\section{Conclusion}
\label{sec:conclusion}

We analyzed SDS and SDI as particle variational inference methods and identified their limiting distributions. With no dependence on the current ensemble, SDS drives every particle to a local maximum of the target density. In contrast, the DDIM inversion in SDI provides an approximate reference to the ensemble, preventing collapse but still leaving the particles with a smaller spread than the target. The same view explains the sign flip in the CFG scale that SDI requires. Replacing the posterior mean in SDI with a reverse PF-ODE solve removes this bias at the cost of an additional ODE solve from every node. We then derived a gradient expression for this dual-bridge extension and used it to motivate PFD, which outperforms the extension in both our analysis and experiments and is competitive with VSD.

%% file: sections/A_proof_prop_sds.tex
\section{Proof of Proposition~\ref{prop:sds}}
\label{app:sds_proof}

With $p_0 = \mathcal{N}(\bm{\mu},\rho^2\bm{I}_d)$ and the VE schedule $\sigma_t = t$, the perturbation is given by: $\bm{x}_t = \bm{x}_0 + t\bm{\epsilon}$. Convolution of two Gaussians gives $p_t = \mathcal{N}(\bm{\mu},(\rho^2+t^2)\bm{I}_d)$, whose score is
\begin{equation}
\nabla_{\bm{x}}\log p_t(\bm{x}_t) = -\frac{\bm{x}_t-\bm{\mu}}{\rho^2+t^2}.
\label{eqn:25}
\end{equation}
Substituting \eqref{eqn:25} into \eqref{eqn:13}, we get
\begin{equation}
\frac{\mathrm{d}\bm{x}_0}{\mathrm{d}\tau} = -\,\mathbb{E}_{t,\bm{\epsilon}}\left[ \frac{t^2\,(\bm{x}_0+t\bm{\epsilon}-\bm{\mu})}{\rho^2+t^2} \right].
\label{eqn:26}
\end{equation}
The inner expectation is $\mathbb{E}_{\bm{\epsilon}}[\bm{x}_0+t\bm{\epsilon}-\bm{\mu}] = \bm{x}_0-\bm{\mu}$, and since the remaining factor depends only on $t$, \eqref{eqn:26} reduces to
\begin{equation}
\frac{\mathrm{d}\bm{x}_0}{\mathrm{d}\tau} = -A(\rho)\,(\bm{x}_0-\bm{\mu}), \qquad A(\rho) = \mathbb{E}_t\!\left[\frac{t^2}{\rho^2+t^2}\right].
\label{eqn:27}
\end{equation}
With $t\sim\mathcal{U}[0,1]$, writing $\frac{t^2}{\rho^2+t^2} = 1 - \frac{\rho^2}{\rho^2+t^2}$, and using $\int_0^1 (\rho^2+t^2)^{-1}\mathrm{d}t = \rho^{-1}\arctan(1/\rho)$ gives
\begin{equation}
A(\rho) = 1 - \rho^2\cdot\frac{1}{\rho}\arctan\frac{1}{\rho} = 1 - \rho\arctan\frac{1}{\rho},
\label{eqn:28}
\end{equation}
which is $>0$ and $\leq 1$ for all $\rho > 0$. With $q_{0,\tau} = \mathcal{N}(\bm{m}_\tau,\nu_{\tau}^2\bm{I}_d)$, taking the expectation with respect to $\bm{x}_0$ gives $\mathrm{d}\bm{m}_\tau/\mathrm{d}\tau = -A(\rho)(\bm{m}_\tau-\bm{\mu})$, and subtracting this from \eqref{eqn:27} gives $\mathrm{d}(\bm{x}_0-\bm{m}_\tau)/\mathrm{d}\tau = -A(\rho)(\bm{x}_0-\bm{m}_\tau)$. Thus, the deviation from the mean contracts at the same rate; therefore, $\mathrm{d}\nu_\tau/\mathrm{d}\tau = -A(\rho)\nu_\tau$. Both are linear, with solutions
\begin{equation}
\bm{m}_\tau = \bm{\mu} + e^{-A(\rho)\tau}(\bm{m}_0-\bm{\mu}), \qquad \nu_\tau^2 = \nu_0^2\,e^{-2A(\rho)\tau}.
\label{eqn:29}
\end{equation}
Since $A(\rho)>0$, both $\|\bm{m}_\tau-\bm{\mu}\|$ and $\nu_\tau$ decay exponentially to zero, so $q_{0,\tau}\to\delta_{\bm{\mu}}$ weakly as $\tau\to\infty$. \qed

%% file: sections/B_proof_prop_sdi.tex
\section{Proof of Proposition~\ref{prop:sdi}}
\label{app:sdi_proof}

We write $\nu$ for $\nu_\tau$ and $\bm{m}$ for $\bm{m}_\tau$ where no confusion arises. The PF-ODE associated with $\{q_{t,\tau}\}_t$ is $\mathrm{d}\bm{x}_t/\mathrm{d}t = -t\,\nabla_{\bm{x}}\log q_{t,\tau}(\bm{x}_t) = t(\bm{x}_t-\bm{m})/(\nu^2+t^2)$, whose solution is given by:
\begin{equation}
\bm{\Phi}^{q_\tau}_{0\to t}(\bm{x}_0) = \bm{m} + c_t(\nu)\,(\bm{x}_0-\bm{m}), \qquad c_t(\nu) \coloneqq \frac{\sqrt{\nu^2+t^2}}{\nu}.
\label{eqn:30}
\end{equation}
The pushforward has variance $\nu^2 c_t(\nu)^2 = \nu^2+t^2$, confirming $(\bm{\Phi}^{q_\tau}_{0\to t})_{\#}q_{0,\tau} = q_{t,\tau}$. Substituting the score of $p_t$ (\eqref{eqn:25}) into the posterior mean $\hat{\bm{x}}_0$ gives
\begin{equation}
\hat{\bm{x}}_0(\bm{x}_t) = \bm{x}_t - \frac{t^2(\bm{x}_t-\bm{\mu})}{\rho^2+t^2} = \bm{\mu} + \frac{\rho^2}{\rho^2+t^2}\,(\bm{x}_t-\bm{\mu}),
\label{eqn:31}
\end{equation}
which is affine, as is \eqref{eqn:30}. Defining $\bm{u} \coloneqq \bm{x}_0-\bm{m}$ and combining \eqref{eqn:30} with \eqref{eqn:31},
\begin{equation}
\hat{\bm{x}}_0\!\left(\bm{\Phi}^{q_\tau}_{0\to t}(\bm{x}_0)\right) - \bm{x}_0 = -\frac{t^2}{\rho^2+t^2}\,(\bm{m}-\bm{\mu}) + \left(\frac{\rho^2\,c_t(\nu)}{\rho^2+t^2} - 1\right)\bm{u}.
\label{eqn:32}
\end{equation}
Taking the expectation over $t$ in \eqref{eqn:17} gives
\begin{equation}
\frac{\mathrm{d}\bm{x}_0}{\mathrm{d}\tau} = -A(\rho)\,(\bm{m}-\bm{\mu}) + \left(B(\nu)-1\right)\bm{u},
\label{eqn:33}
\end{equation}
with $A(\rho) = \mathbb{E}_t[t^2/(\rho^2+t^2)]$ and $B(\nu) = (\rho^2/\nu)\,\mathbb{E}_t[\sqrt{\nu^2+t^2}/(\rho^2+t^2)]$. Since $\mathbb{E}[\bm{u}]=\bm{0}$, taking the expectation with respect to $\bm{x}_0$ in \eqref{eqn:33} gives \eqref{eqn:18a}, and subtracting this from \eqref{eqn:33} leaves $\mathrm{d}\bm{u}/\mathrm{d}\tau = (B(\nu)-1)\bm{u}$. Thus, the deviation from the mean contracts at a rate of $1-B(\nu)$, and \eqref{eqn:18b} follows.

Writing
\begin{equation}
B(\nu) = \rho^2\,\mathbb{E}_t\!\left[\frac{1}{\rho^2+t^2}\sqrt{1+\frac{t^2}{\nu^2}}\right]=\rho^2\int_0^1\frac{1}{\rho^2+t^2}\sqrt{1+\frac{t^2}{\nu^2}}\mathrm{d}t,
\label{eqn:34}
\end{equation}
the integrand is strictly decreasing in $\nu$ for every $t>0$, so $B$ is strictly decreasing on $(0,\infty)$; it diverges as $\nu\to0^+$ and tends to $\rho^2\,\mathbb{E}_t[(\rho^2+t^2)^{-1}] = 1-A(\rho)$ as $\nu\to\infty$. Since $1-A(\rho)<1$, there is a unique $\nu_\star$ with $B(\nu_\star)=1$. For $\nu<\nu_\star$, we have $B(\nu)>1$ and $\mathrm{d}\nu_\tau/\mathrm{d}\tau>0$, and for $\nu>\nu_\star$ the reverse, so $\nu_\star$ is globally attracting. Finally, evaluating \eqref{eqn:31} at $\nu=\rho$ gives $B(\rho) = \rho\,\mathbb{E}_t[(\rho^2+t^2)^{-1/2}] < 1$, because $(\rho^2+t^2)^{-1/2}<\rho^{-1}$ for $t>0$. As $B$ is decreasing and $B(\nu_\star)=1$, it follows that $\nu_\star<\rho$, and the particles converge to $\mathcal{N}(\bm{\mu},\nu_\star^2\bm{I}_d)$. \qed

%% file: sections/C_proof_prop_pfd.tex
\section{Proof of Proposition~\ref{prop:pfd}}
\label{app:pfd_proof}

Evaluating the score difference at a point $\bm{x}_t$, we obtain
\begin{equation}
\nabla_{\bm{x}}\log p_t(\bm{x}_t) - \nabla_{\bm{x}}\log q_{t,\tau}(\bm{x}_t)= -\frac{\bm{m}-\bm{\mu}}{\rho^2+t^2} + \frac{\rho^2-\nu^2}{(\nu^2+t^2)(\rho^2+t^2)}\,(\bm{x}_t-\bm{m}) .
\label{eqn:35}
\end{equation}
Taking $\eta=1$ in \eqref{eqn:22a} and \eqref{eqn:22b}, and appending the Jacobian factor for DB, the three updates share the same form.
\begin{equation}
\left.\frac{\mathrm{d}\bm{x}_0}{\mathrm{d}\tau}\right|_{\mathrm{M}} = \int_0^1 w(t)\,\Lambda^{\mathrm{M}}_t \left(\nabla_{\bm{x}}\log p_t(\bm{x}^{\mathrm{M}}_t) - \nabla_{\bm{x}}\log q_{t,\tau}(\bm{x}^{\mathrm{M}}_t)\right)\mathrm{d}t ,
\label{eqn:36}
\end{equation}
and differ only in $\Lambda^{\mathrm{M}}_t$ and in the coupling $\bm{x}^{\mathrm{M}}_t$. PFD has $\Lambda_t = 1$ with $\bm{x}_t = \bm{\Phi}^{q_\tau}_{0\to t}(\bm{x}_0)$, so $\bm{x}_t - \bm{m} = c_t(\nu)\,\bm{u}$ by \eqref{eqn:30}, VSD has $\Lambda_t = 1$ with $\bm{x}_t = \bm{x}_0+t\bm{\epsilon}$, for which $\mathbb{E}_{\bm{\epsilon}}[\bm{x}_t-\bm{m}] = \bm{u}$, and DB has the same coupling as PFD with
\begin{equation}
\Lambda_t = \frac{\partial\bm{\Phi}^{p}_{t\to0}}{\partial\bm{x}} = \frac{\rho}{\sqrt{\rho^2+t^2}} ,
\label{eqn:37}
\end{equation}
obtained as in \eqref{eqn:30} but for the reverse PF-ODE, for which we get $\bm{\Phi}^{p}_{0\to t}(\bm{x}) - \bm{\mu} = (\sqrt{\rho^2+t^2}/\rho)\,(\bm{x}-\bm{\mu})$. Substituting \eqref{eqn:35} into \eqref{eqn:36}, every update is affine in $\bm{x}_0$, so the isotropic Gaussian family is closed under all three. Since $\mathbb{E}[\bm{u}]=\bm{0}$, taking the expectation with respect to $\bm{x}_0$ gives the left-hand side of \eqref{eqn:24} with
\begin{equation}
A_{\mathrm{PFD}} = A_{\mathrm{VSD}} = \int_0^1 \frac{w(t)}{\rho^2+t^2}\,\mathrm{d}t, \qquad A_{\mathrm{DB}} = \rho\int_0^1 \frac{w(t)}{(\rho^2+t^2)^{3/2}}\,\mathrm{d}t,
\label{eqn:38}
\end{equation}
and subtracting it from \eqref{eqn:36} leaves $\mathrm{d}\bm{u}/\mathrm{d}\tau$, whose coefficient gives  $\mathrm{d}\nu_\tau/\mathrm{d}\tau$, as in Section~\ref{prop:sdi}. The factor $\rho^2-\nu^2$ in \eqref{eqn:35} does not depend on $t$. Pulling it out and using $\nu\,c_t(\nu) = \sqrt{\nu^2+t^2}$ gives the right-hand side of \eqref{eqn:24} with
\begin{subequations}
\begin{align}
B_{\mathrm{PFD}}(\nu) &= \int_0^1 \frac{w(t)}{\sqrt{\nu^2+t^2}\,(\rho^2+t^2)}\,\mathrm{d}t, 
\label{eqn:39a}\\[2pt] 
B_{\mathrm{VSD}}(\nu) &= \int_0^1 \frac{w(t)\,\nu}{(\nu^2+t^2)(\rho^2+t^2)}\,\mathrm{d}t, 
\label{eqn:39b}\\[2pt] 
B_{\mathrm{DB}}(\nu) &= \int_0^1 \frac{w(t)\,\rho}{\sqrt{\nu^2+t^2}\,(\rho^2+t^2)^{3/2}}\,\mathrm{d}t .
\label{eqn:39c}
\end{align}
\end{subequations}
Since $w>0$, all three integrands are positive on $(0,1)$ for every $\nu>0$, so $B_{\mathrm{M}}(\nu)>0$ and the sign of $\mathrm{d}\nu_\tau/\mathrm{d}\tau$ is that of $\rho^2-\nu_\tau^2$: the variance increases below $\rho$ and decreases above it, so $\nu_\star=\rho$ is the unique stationary point and is globally attracting. Likewise $A_{\mathrm{M}}>0$, so $\bm{m}_\tau\to\bm{\mu}$ exponentially, and $p_0$ is the unique, globally attracting fixed point of all three. 

It remains to order the constants. The integrands of \eqref{eqn:38} differ by
$\rho/\sqrt{\rho^2+t^2}<1$, so
\begin{equation}
A_{\mathrm{DB}} < A_{\mathrm{PFD}} = A_{\mathrm{VSD}} .
\label{eqn:40}
\end{equation}
Those of \eqref{eqn:39a}, \eqref{eqn:39b} and \eqref{eqn:39c} differ pointwise in $t$ by
\begin{equation}
\frac{B_{\mathrm{PFD}}}{B_{\mathrm{VSD}}}:\;\frac{\sqrt{\nu^2+t^2}}{\nu}, \qquad \frac{B_{\mathrm{PFD}}}{B_{\mathrm{DB}}}:\;\frac{\sqrt{\rho^2+t^2}}{\rho}, \qquad \frac{B_{\mathrm{DB}}}{B_{\mathrm{VSD}}}:\;\frac{\rho\sqrt{\nu^2+t^2}}{\nu\sqrt{\rho^2+t^2}},
\label{eqn:41}
\end{equation}
where the first two exceed one for every $t>0$ and every $\nu>0$, giving $B_{\mathrm{PFD}}(\nu) > B_{\mathrm{VSD}}(\nu)$ and $B_{\mathrm{PFD}}(\nu) > B_{\mathrm{DB}}(\nu)$. The third exceeds one for $\nu<\rho$ and falls below one for $\nu>\rho$, the two coinciding only at $\nu=\rho$; there is, therefore, no ordering between $B_{\mathrm{VSD}}$ and $B_{\mathrm{DB}}$. For isotropic Gaussians, we have
\begin{equation}
D_{\mathrm{KL}}(q_{0,\tau}\,\|\,p_0)
= \frac{\|\bm{m}_\tau-\bm{\mu}\|^2}{2\rho^2}
+ \frac{d}{2}\left(\frac{\nu_\tau^2}{\rho^2} - 1 - 2\log\frac{\nu_\tau}{\rho}\right),
\label{eqn:42}
\end{equation}
which increases as $\bm{m}_\tau$ moves away from $\bm{\mu}$ and as $\nu_\tau$ moves away from $\rho$ in either direction. PFD drives $\bm{m}_\tau$ to $\bm{\mu}$ at least as fast as the other two by \eqref{eqn:40} and $\nu_\tau$ to $\rho$ strictly faster by \eqref{eqn:41}, so it decreases $D_{\mathrm{KL}}(q_{0,\tau}\,\|\,p_0)$ strictly faster than both VSD and DB. \qed

%% file: sections/D_proof_thm_db.tex
\section{Proof of Theorem~\ref{thm:db}}
\label{app:db_proof}
Writing $\bm{v}^q(\bm{x}_t,t) = \bm{f}(\bm{x}_t,t) - \tfrac12 g_t^2\nabla_{\bm{x}}\log q_{t,\tau}(\bm{x}_t)$ and $\bm{v}^p(\bm{x}_t,t) = \bm{f}(\bm{x}_t,t) - \tfrac12 g_t^2\nabla_{\bm{x}} \log p_t(\bm{x}_t)$ for the velocity fields of the forward and reverse PF-ODEs, respectively. Now, $\bm{\Delta}(\bm{x}_0,s) = \bm{x}_0 - \bm{\Phi}^{p}_{s\to0}(\bm{x}_s)$; taking the derivative with respect to $s$ gives
\begin{equation}
\frac{\mathrm{d}\bm{\Delta}}{\mathrm{d}t} = -\frac{\mathrm{d}}{\mathrm{d}s}\,\bm{\Phi}^{p}_{s\to0}(\bm{x}_s) = -\left( \frac{\partial\bm{\Phi}^{p}_{s\to0}}{\partial\bm{x}}(\bm{x}_s)\,\bm{v}^q(\bm{x}_s,s) + \partial_s\bm{\Phi}^{p}_{s\to0}(\bm{x}_s) \right),
\label{eqn:43}
\end{equation}
where the first term comes from the dependence on $\bm{x}_s$, with $\mathrm{d}\bm{x}_s/\mathrm{d}s = \bm{v}^q(\bm{x}_s,s)$, and the second from the $s$-dependence of $\bm{\Phi}^{p}_{s\to0}$.

\begin{lemma}[Derivative of a flow map in its initial time]
Let $\Phi_{s\to t}(\bm{x})$ be the flow map of $\mathrm{d}\bm{x}/\mathrm{d}\tau = \bm{F}(\bm{x},\tau)$, carrying a state $\bm{x}$ at time $s$ to time $t$. Then we have the following:
\begin{equation}
\partial_s\Phi_{s\to t}(\bm{x}) = -\frac{\partial\Phi_{s\to t}}{\partial\bm{x}}\,\bm{F}(\bm{x},s),
\label{eqn:44}
\end{equation}
where $\partial\Phi_{s\to t}/\partial\bm{x}$ is the Jacobian of the flow map with respect to its initial condition, evaluated at the starting point $\bm{x}$. See Chapter~5 of \citet{ode} for a proof.
\label{lem:flow}
\end{lemma}

Applying Lemma~\ref{lem:flow} with $\bm{F}=\bm{v}^p$ to the second term of \eqref{eqn:43}, we get
\begin{equation}
\begin{split}
\frac{\mathrm{d}\bm{\Delta}}{\mathrm{d}s} &= -\frac{\partial\bm{\Phi}^{p}_{s\to0}}{\partial\bm{x}}(\bm{x}_s) \left( \bm{v}^q(\bm{x}_s,s) - \bm{v}^p(\bm{x}_s,s) \right)
\\[2pt] &=
\frac{1}{2}g_s^2\, \frac{\partial\bm{\Phi}^{p}_{s\to0}}{\partial\bm{x}}(\bm{x}_s) \left( \nabla_{\bm{x}}\log q_{s,\tau}(\bm{x}_s) - \nabla_{\bm{x}}\log p_s(\bm{x}_s) \right),
\end{split}
\label{eqn:45}
\end{equation}
where the drift $\bm{f}$ cancels out in the difference. Integrating \eqref{eqn:45} from $0$ to $t$ subject to the boundary condition $\bm{\Delta}(\bm{x}_0,0) = \bm{0}$ gives \eqref{eqn:20a}. By Fubini's theorem, exchanging the order of integration with $t\sim\mathcal{U}[0,1]$ gives
\begin{equation}
\mathbb{E}_t[\bm{\Delta}(\bm{x}_0,t)] = \int_0^1\!\!\int_0^t (\cdot)_s\,\mathrm{d}s\,\mathrm{d}t = \int_0^1 (\cdot)_s\!\int_s^1\!\mathrm{d}t\,\mathrm{d}s = \int_0^1 (1-s)\,(\cdot)_s\,\mathrm{d}s,
\label{eqn:46}
\end{equation}
where $(\cdot)_s$ denotes the integrand of \eqref{eqn:20a}, which gives \eqref{eqn:20b}. \qed

%% file: sections/E_proof_thm_pfd.tex
\section{Proof of Theorem~\ref{thm:pfd}}
\label{app:pfd_small_time}
Define $\bm{G}_t(\bm{x}) \coloneqq \nabla_{\bm{x}}\log p_t(\bm{x}) - \nabla_{\bm{x}}\log q_{t,\tau}(\bm{x})$, then the two velocity fields are given by:
\begin{equation}
\bm{v}_{\mathrm{VSD}}(\bm{x}_0) = \int_0^\eta w(t)\,\mathbb{E}_{\bm{\epsilon}}\!\left[\bm{G}_t(\bm{x}_0+t\bm{\epsilon})\right]\mathrm{d}t, \qquad \bm{v}_{\mathrm{PFD}}(\bm{x}_0) = \int_0^\eta w(t)\,\bm{G}_t\!\left(\Phi^{q_\tau}_{0\to t}(\bm{x}_0)\right)\mathrm{d}t.
\label{eqn:47}
\end{equation}
If the particles evolve under a field $\bm{v}$, the density satisfies the continuity equation $\partial_\tau q_{0,\tau} = -\nabla\cdot(q_{0,\tau}\bm{v})$; hence, differentiating $D_{\mathrm{KL}}(q_{0,\tau}\|p_0)$ and integrating by parts gives
\begin{equation}
\frac{\mathrm{d}D_{\mathrm{KL}}}{\mathrm{d}\tau} = \int \log\frac{q_{0,\tau}}{p_0}\,\partial_\tau q_{0,\tau} = -\,\mathbb{E}_{\bm{x}_0\sim q_{0,\tau}}\!\left[\left\langle \bm{v}(\bm{x}_0),\,\bm{G}_0(\bm{x}_0)\right\rangle\right].
\label{eqn:48}
\end{equation}
Thus, it suffices to expand $\bm{v}_{\mathrm{VSD}}-\bm{v}_{\mathrm{PFD}}$ and pair the result with $\bm{G}_0$. For the first field, a second-order Taylor expansion of $\bm{G}_t$ about $\bm{x}_0$ gives
\begin{equation}
\bm{G}_t(\bm{x}_0+t\bm{\epsilon})
=
\bm{G}_t(\bm{x}_0)
+ t\,\frac{\partial\bm{G}_t}{\partial\bm{x}}(\bm{x}_0)\,\bm{\epsilon}
+ \frac{t^2}{2}\,\bm{Q}_t(\bm{x}_0)[\bm{\epsilon}]
+ O(t^3),
\label{eqn:49}
\end{equation}
where $\bm{Q}_t(\bm{x}_0)[\bm{\epsilon}]$ is the vector whose $j$-th entry is $\bm{\epsilon}^{\!\top}\bm{H}_j\bm{\epsilon}$, with $\bm{H}_j$ being the Hessian of the $j$-th component of $\bm{G}_t$ at $\bm{x}_0$. The first-order term vanishes under $\mathbb{E}_{\bm{\epsilon}}[\bm{\epsilon}]=\bm{0}$, and since  $\mathbb{E}_{\bm{\epsilon}}[\bm{\epsilon}^{\!\top}\bm{H}_j\bm{\epsilon}] = \operatorname{tr}\bm{H}_j = \Delta\,[\bm{G}_t]_j$, the quadratic term contracts into a component-wise Laplacian $\Delta\bm{G}_t(\bm{x}_0)$,
\begin{equation}
\mathbb{E}_{\bm{\epsilon}}\!\left[\bm{G}_t(\bm{x}_0+t\bm{\epsilon})\right] = \bm{G}_t(\bm{x}_0) + \frac{t^2}{2}\,\Delta\bm{G}_t(\bm{x}_0) + O(t^4),
\label{eqn:50}
\end{equation}
the odd orders vanish due to the symmetry of the Gaussian. For the second field, the forward PF-ODE reads $\mathrm{d}\bm{x}/\mathrm{d}s = -s\,\nabla_{\bm{x}}\log q_{s,\tau}(\bm{x})$, so $\Phi^{q_\tau}_{0\to t}(\bm{x}_0) = \bm{x}_0 - \int_0^t s\, \nabla_{\bm{x}}\log q_{s,\tau}(\bm{x}_s)\,\mathrm{d}s$, giving
\begin{equation}
\Phi^{q_\tau}_{0\to t}(\bm{x}_0) = \bm{x}_0 - \frac{t^2}{2}\,\nabla_{\bm{x}}\log q_{0,\tau}(\bm{x}_0) + O(t^4),
\label{eqn:51}
\end{equation}
and a first-order expansion of $\bm{G}_t$ about $\bm{x}_0$ then gives
\begin{equation}
\bm{G}_t\!\left(\Phi^{q_\tau}_{0\to t}(\bm{x}_0)\right) = \bm{G}_t(\bm{x}_0) - \frac{t^2}{2}\,\frac{\partial\bm{G}_t}{\partial\bm{x}}(\bm{x}_0)\, \nabla_{\bm{x}}\log q_{0,\tau}(\bm{x}_0) + O(t^4).
\label{eqn:52}
\end{equation}
Both expansions share the zeroth-order term $\bm{G}_t(\bm{x}_0)$, which cancels exactly in the difference. What survives is the difference of the two $t^2$ terms, with a remainder of $O(t^4)$. Subtracting and integrating against $w(t)$, with $C_\eta = \tfrac12\int_0^\eta w(t)\,t^2\,\mathrm{d}t$ and $\int_0^\eta w(t)\,t^4\,\mathrm{d}t \le \eta^2\int_0^\eta w(t)\,t^2\,\mathrm{d}t$, results in
\begin{equation}
\bm{v}_{\mathrm{VSD}}-\bm{v}_{\mathrm{PFD}} = C_\eta \left( \Delta\bm{G}_0 + \frac{\partial\bm{G}_0}{\partial\bm{x}}\,\nabla_{\bm{x}}\log q_{0,\tau} + O(\eta^2) \right).
\label{eqn:53}
\end{equation}
Write $\varphi \coloneqq \log p_0 - \log q_{0,\tau}$, so that $\bm{G}_0 = \nabla\varphi$ and $\partial\bm{G}_0/\partial\bm{x} = \nabla^2\varphi$, and let $\mathcal{L}$ denote the infinitesimal
Langevin operator of $q_{0,\tau}$,
\begin{equation}
\mathcal{L}f = \Delta f + \nabla f^{\!\top}\nabla\log q_{0,\tau} .
\label{eqn:54}
\end{equation}
The bracket in \eqref{eqn:53} is then $\mathcal{L}$ applied to each component of $\bm{G}_0$, since its $j$-th entry is $\Delta\,\partial_j\varphi + \nabla(\partial_j\varphi)^{\!\top}\nabla\log q_{0,\tau} = \mathcal{L}\,\partial_j\varphi$. Note that this result can be generalized for any smooth test function, we state it as a result in its own right.
\begin{proposition}
Under the VE schedule $\sigma_t = t$, let $f:\mathbb{R}^d\to\mathbb{R}$ be $C^4$ and $\bm{\epsilon}\sim\mathcal{N}(\bm{0},\bm{I}_d)$. Then
\begin{equation}
\mathbb{E}_{\bm{\epsilon}}\big[f(\bm{x}_0 + t\bm{\epsilon})\big] - f\big(\bm{\Phi}^{q_\tau}_{0\to t}(\bm{x}_0)\big) = \frac{t^2}{2}\,\mathcal{L}f(\bm{x}_0) + O(t^4),
\label{eqn:55}
\end{equation}
where $\mathcal{L}$, defined in \eqref{eqn:54}, is the infinitesimal generator of the Langevin dynamics with invariant density $q_{0,\tau}$.
\label{prop:coupling_gap}
\end{proposition}
Multiplying \eqref{eqn:54} by $q_{0,\tau}$ and using $q_{0,\tau}\nabla\log q_{0,\tau} = \nabla q_{0,\tau}$ gives $q_{0,\tau}\mathcal{L}f = q_{0,\tau}\Delta f + \nabla f^{\!\top}\nabla q_{0,\tau}$, which, by the product rule, is
\begin{equation}
q_{0,\tau}\,\mathcal{L}f = \nabla\cdot\left(q_{0,\tau}\nabla f\right),
\label{eqn:56}
\end{equation}
so that integrating by parts and discarding the boundary term because $q_{0,\tau}$ decays,
\begin{equation}
-\,\mathbb{E}_{\bm{x}_0\sim q_{0,\tau}}\!\left[f\,\mathcal{L}f\right] =
-\int f\,\nabla\cdot\left(q_{0,\tau}\nabla f\right) =
\int q_{0,\tau}\,\|\nabla f\|^2 =
\mathbb{E}_{\bm{x}_0\sim q_{0,\tau}}\!\left[\|\nabla f\|^2\right],
\label{eqn:57}
\end{equation}
the Dirichlet energy associated with $\mathcal{L}$. Applying \eqref{eqn:57} with $f = \partial_j\varphi$, summing over $j$, and using $\sum_j\|\nabla\partial_j\varphi\|^2 = \|\nabla^2\varphi\|_F^2$,
\begin{equation}
\mathbb{E}_{\bm{x}_0\sim q_{0,\tau}}\!\left[\left\langle \Delta\bm{G}_0 + \frac{\partial\bm{G}_0}{\partial\bm{x}}\,\nabla\log q_{0,\tau},\;\bm{G}_0 \right\rangle\right] = -\,\mathbb{E}_{\bm{x}_0\sim q_{0,\tau}}\!\left[\left\|\nabla^2\varphi\right\|_F^2\right].
\label{eqn:58}
\end{equation}
Substituting \eqref{eqn:53} into \eqref{eqn:48} and applying \eqref{eqn:58},
\begin{equation}
\left.\frac{\mathrm{d}D_{\mathrm{KL}}}{\mathrm{d}\tau}\right|_{\mathrm{VSD}} - \left.\frac{\mathrm{d}D_{\mathrm{KL}}}{\mathrm{d}\tau}\right|_{\mathrm{PFD}} = C_\eta\left( \mathbb{E}_{\bm{x}_0\sim q_{0,\tau}}\!\left[\left\|\nabla^2\varphi\right\|_F^2\right] + O(\eta^2)\right),
\label{eqn:59}
\end{equation}
with $\nabla^2\varphi = \nabla^2\log p_0 - \nabla^2\log q_{0,\tau}$. \qed

%% file: sections/F_ddim_inv.tex
\section{Negative CFG scale in image editing}
\label{app:ddim_editing}

Figure~\ref{fig:ddim_edit} compares three ways of running the inversion, each followed by DDIM sampling under the target prompt. Method~1 is standard DDIM inversion editing, which inverts under a source prompt describing the input at a positive CFG scale. Once guidance is applied, it no longer preserves the input faithfully, which is what motivates methods such as Null-Text Inversion~\citep{nulltextinv}. Method~2 replaces the source prompt with the target prompt at the same positive scale. As discussed in Section~\ref{subsec:sdi}, this tilts the inversion towards $p_t$ rather than away from it, and the edits become unstable. Method~3 inverts under the target prompt with the sign reversed, as SDI does. By the argument of Section~\ref{subsec:sdi}, this approximates the forward PF-ODE induced by $\{q_{t,\tau}\}_t$, so the edit becomes an approximate DDIB between the input and target distributions. Method~3 restores the stability lost in Method~2 at the cost of some over-saturation and smoothing. This supports reading DDIM inversion under the target prompt with a CFG negative scale as an approximation of $\bm{\Phi}^{q_\tau}_{0\to t}$.

\begin{figure}[ht]
\centering
\includegraphics[width=0.7\linewidth]{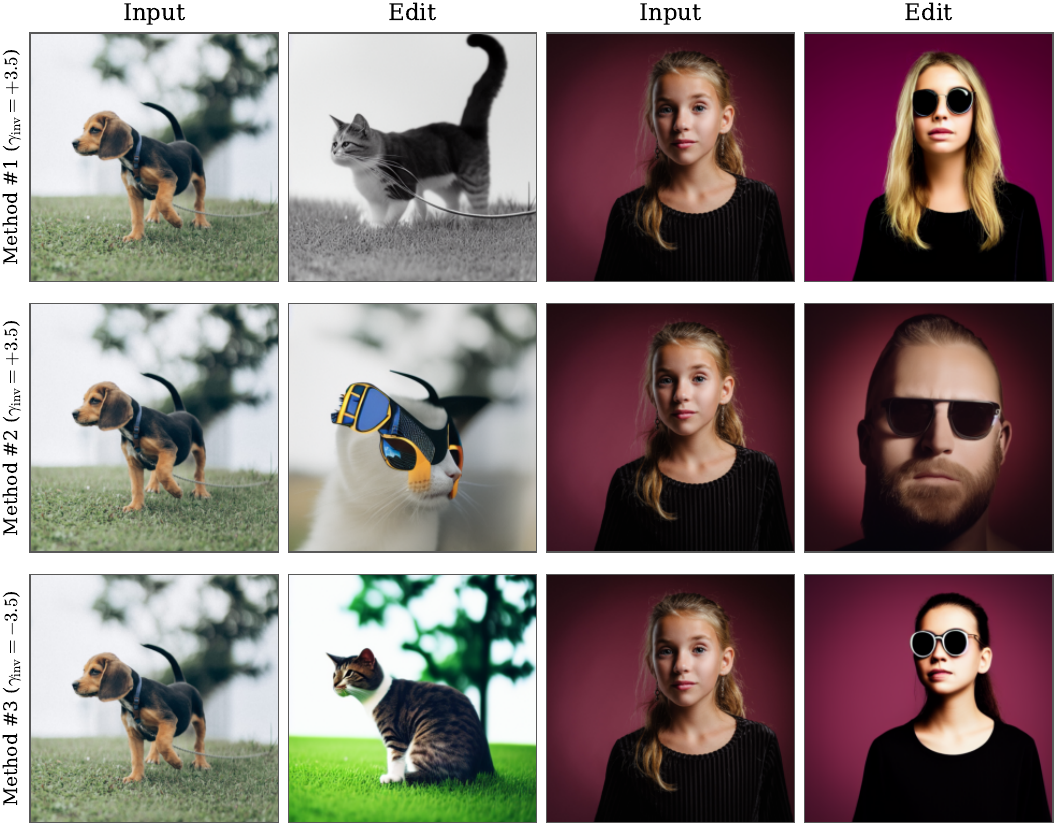}
\caption{Editing under three different inversion methods. Method~1 inverts under a source prompt describing the input at $\gamma_{\mathrm{inv}}=+3.5$; Method~2 inverts under the target prompt at the same positive scale; Method~3 inverts under the target prompt at $\gamma_{\mathrm{inv}}=-3.5$. The source prompts are \emph{``photograph of a puppy on the grass''} and \emph{``a photograph of a face''}; the target prompts are \emph{``photograph of a cat on the grass''} and \emph{``a photograph of a face with sunglasses''}. The two input images are taken from \url{https://huggingface.co/learn/diffusion-course/en/unit4/2}, and the \texttt{stable-diffusion-v1-5} checkpoint is used.}
\label{fig:ddim_edit}
\end{figure}

%% file: sections/G_grouping.tex
\section{Grouping of existing distillation methods}
\label{app:grouping}

Figure~\ref{fig:grouping} shows the classification of some text-to-3D distillation methods as instances of the three methods. VSD uses $\bm{\epsilon}_{\phi}(\bm{x}_t,t,c)-\bm{\epsilon}_{\psi,\tau}(\bm{x}_t,t)$ as the gradient direction, and Classifier-Score Distillation (CSD)~\citep{csd} approximates this in the same way that SDI approximates its ideal implementation (See Section~\ref{subsec:sdi}), by replacing the noise predictor corresponding to the current particle distribution $\bm{\epsilon}_{\psi,\tau}(\bm{x}_t,t)$ with the unconditional predictor so that its gradient direction is $\bm{\epsilon}_\phi(\bm{x}_t,t,c) -\bm{\epsilon}_\phi(\bm{x}_t,t,\varnothing)$. Noise-Free Score Distillation (NFSD) can likewise be viewed as a VSD approximation (See Section 5 of~\citet{nfsd}). SDS-Bridge~\citep{sdsbridge} alludes to the dual-bridge construction of Section~\ref{subsec:dual_bridge} but avoids it in practice on the grounds of cost. Its algorithm instead runs vanilla SDS and then uses the gradient estimator $\bm{\epsilon}_\phi(\bm{x}_t,t,c)-\bm{\epsilon}_\phi(\bm{x}_t,t,c_{\mathrm{src}})$, where the source prompt $c_{\mathrm{src}}$ (e.g., \texttt{blurry, over-saturated}) is intended as a closer approximation of the current noise prediction than the unconditional one. Hence, it is an approximation of VSD as well.

SDS-based methods use $\hat{\bm{\epsilon}}(\bm{x}_t,t,c)-\bm{\epsilon}$ as their gradient direction. Their official implementations, however, use a large CFG scale $\gamma$, as this provides better results. Expanding the guided prediction $\hat{\bm{\epsilon}}(\bm{x}_t,t,c)$ gives:
\begin{equation} 
\left(\bm{\epsilon}_\phi(\bm{x}_t,t,\varnothing)-\bm{\epsilon}\right) +\gamma\left(\bm{\epsilon}_\phi(\bm{x}_t,t,c) -\bm{\epsilon}_\phi(\bm{x}_t,t,\varnothing)\right),
\label{eqn:60}
\end{equation}
in which the first term does not grow with $\gamma$, while the second is the CSD-approximation of VSD. At the scales at which these methods are run, the second term dominates. Thus, at large $\gamma$, they are also approximations of VSD. Therefore, in our text-to-3D experiment, we stick to a low CFG value ($7.5$) in SDS.

Interval Score Matching (ISM)~\citep{ism}, like SDI, constructs $\bm{x}_t$ by DDIM inversion, and the two differ in how that inversion is run. ISM inverts under the null prompt alone. SDI, however, uses both the conditional and the unconditional predictions; therefore, it needs a negative guidance scale to exchange the roles of the current particle and target distributions, as argued in Section~\ref{subsec:sdi}. They also differ in the reference noise: SDI uses $\bm{\epsilon}^\star = (\bm{x}_t-\bm{x}_0)/\sigma_t$, whereas ISM uses $\bm{\epsilon}_\phi(\bm{x}_{t-\Delta t},t-\Delta t,\varnothing)$, which is the noise inserted at the final inversion step alone.

\begin{figure}[ht]
\centering
\begin{tikzpicture}[font=\small, fam/.style={draw, rounded corners=3pt, fill=black!10, minimum width=3.0cm, minimum height=0.7cm, align=center}, leaf/.style={align=center, inner sep=2pt, text width=4.0cm}, lk/.style={semithick, black!55}
]
\node[fam] (root) {Text-to-3D Distillation};

\node[fam, below=1cm of root, xshift=-4cm] (sds) {SDS};
\node[fam, below=1cm of root]                (sdi) {SDI};
\node[fam, below=1cm of root, xshift= 4cm] (vsd) {VSD};

\draw[lk] (root.south) -- ++(0,-0.4) -| (sds.north);
\draw[lk] (root.south) -- (sdi.north);
\draw[lk] (root.south) -- ++(0,-0.4) -| (vsd.north);

\node[leaf, below=0.5cm of sds] (sdsl) {SDS (DreamFusion)\\ Magic3D~\citep{magic3d}\\ Fantasia3D~\citep{fantasia3d}\\ HiFA~\citep{hifa}};
\node[leaf, below=0.5cm of sdi] (sdil) {SDI\\ ISM (LucidDreamer)};
\node[leaf, below=0.5cm of vsd] (vsdl) {VSD (ProlificDreamer)\\ CSD\\ NFSD\\ SDS-Bridge};

\draw[lk] (sds.south) -- (sdsl.north);
\draw[lk] (sdi.south) -- (sdil.north);
\draw[lk] (vsd.south) -- (vsdl.north);

\draw[lk, dashed, -{Latex[length=2mm]}] (sdsl.south) -- ++(0,-0.45) -- node[below, font=\footnotesize, black!55] {large $\gamma$} ($(vsdl.south)+(0,-0.45)$) -- (vsdl.south);
\end{tikzpicture}
\caption{Grouping of text-to-3D distillation methods into three families. The dashed edge indicates the VSD approximation of SDS at large guidance scales.}
\label{fig:grouping}
\end{figure}

%% file: sections/H_implementation.tex
\section{Implementation Details}
\label{app:implementation}

All experiments are conducted on a single NVIDIA A100 GPU with 80\,GB of memory.

\subsection{Experiments on the CelebA dataset.}
\label{app:celeb_implementation}

The pretrained model is a U-Net with base width $128$, channel multipliers $(1,2,2,2)$, two residual blocks per resolution, and self-attention at $16\times16$. It uses the preconditioning, log-normal noise distribution ($P_{\mathrm{mean}}=-1.2$, $P_{\mathrm{std}}=1.2$), and $\sigma \in [0.002, 80]$ of \citet{karras}, and is trained on $64\times64$ images for $200$k steps with AdamW (learning rate $2\times10^{-4}$, batch size $64$, EMA decay $0.9999$). The Euler and Heun baselines use the $\rho=7$ schedule of \citet{karras}. For distillation, the integral over noise levels is evaluated using the trapezoidal rule on $32$ nodes of the same schedule over $\sigma \in [0.002, 5]$. The forward PF-ODE of SDI and PFD is solved with one Euler step between consecutive nodes. $\bm{\epsilon}_{\psi,\tau}$ is initialized from the pretrained model and trained with Adam (learning rate $10^{-5}$) for $4$ steps on batches of $32$ particles before each particle update. Particles are initialized from $\mathcal{N}(\bm{0}, \bm{I}_d)$ and optimized for $500$ steps with SGD (learning rate $0.1$), using $128$ particles per task. Image generation is the exception: it uses $2500$ particles per run, and $\bm{\epsilon}_{\psi,\tau}$ takes $8$ steps on batches of $128$. For symmetric generation, $\bm{\theta}$ is the $64\times32$ left half. For inpainting, we use left-half, top-half, and diagonal masks. For vector graphics, each particle has $256$ open paths of $1$--$3$ cubic B\'ezier segments with RGBA stroke colors, rendered at $128\times128$ over a gray background and downsampled to $64\times64$. The control points, stroke colors, and stroke widths (clamped to $[1, 4]$ pixels) are optimized with Adam at learning rates $2.0$, $0.01$, and $0.1$.

\subsection{Text-to-3D.}
\label{app:text-to-3d-implementation}

All methods share the same scene representation, renderer, losses, and optimizer, differing only in the guidance. \texttt{threestudio} supports only the single-particle setting, where only one representation is optimized, so $\pi_\tau = \delta_{\bm{\theta}_\tau}$. The analysis still carries over since $q_{0,\tau} = \mathrm{Law}\big(\bm{g}(\bm{\theta}_\tau, \bm{\omega})\big)$ with $\bm{\omega}\sim\xi$ remains non-degenerate. It is the distribution of renderings of the asset over camera poses, and $\bm{\epsilon}_{\psi,\tau}$ models its score. Its spread is therefore visible as the detail of a single asset. Tying every view to one representation, along with view-dependent prompting, is what keeps those views mutually consistent. However, Janus artifacts persist, but this is orthogonal to our work. The scene is an implicit volume with an Instant-NGP~\citep{instantngp} hash-grid encoding ($16$ levels, $2$ features per level, a $2^{19}$ table, and resolutions from $16$ to $4096$), a softplus density initialized with the blob of \citet{magic3d}, and a neural environment-map background. It is rendered at $512\times512$ with $512$ samples per ray, from random cameras with distance in $[1, 1.5]$, field of view in $[40^\circ, 70^\circ]$, and elevation in $[-10^\circ, 45^\circ]$. We add z-variance regularization~\citep{hifa} with weight $1$ and optimize for $5$k steps with AdamW ($\beta = (0.9, 0.99)$) in mixed precision, using learning rates of $0.01$ for the hash grid and $10^{-3}$ for the density, color, and background networks. SDI, VSD, and PFD use Stable Diffusion 2.1 base (\texttt{stabilityai/stable-diffusion-2-1-base}) with a guidance scale of $7.5$. For VSD and PFD, $\bm{\epsilon}_{\psi,\tau}$ is a LoRA of the same model, trained with a learning rate $10^{-4}$ and a guidance scale $1$ as in \citet{vsd}. SDI follows its original implementation, with DDIM inversion under the target prompt and a negative guidance scale, so that the behavior analyzed in Section~\ref{sec:sds_sdi} carries over. SDS with Stable Diffusion at this resolution and a standard guidance scale collapses to empty space within $500$ steps, so we use DeepFloyd IF (\texttt{DeepFloyd/IF-I-XL-v1.0}) for SDS instead. Evaluation renders $120$ views on a circle at distance $1.5$ with a $70^\circ$ field of view.

%% file: sections/I_qualitative.tex
\section{Qualitative Results}
\label{app:qualitative}

\subsection{Image Generation}
\label{app:generation}

\begin{figure}[H]
\centering
\begin{subfigure}{0.7\textwidth}
\centering
\includegraphics[width=\textwidth,height=0.4\textheight,keepaspectratio]{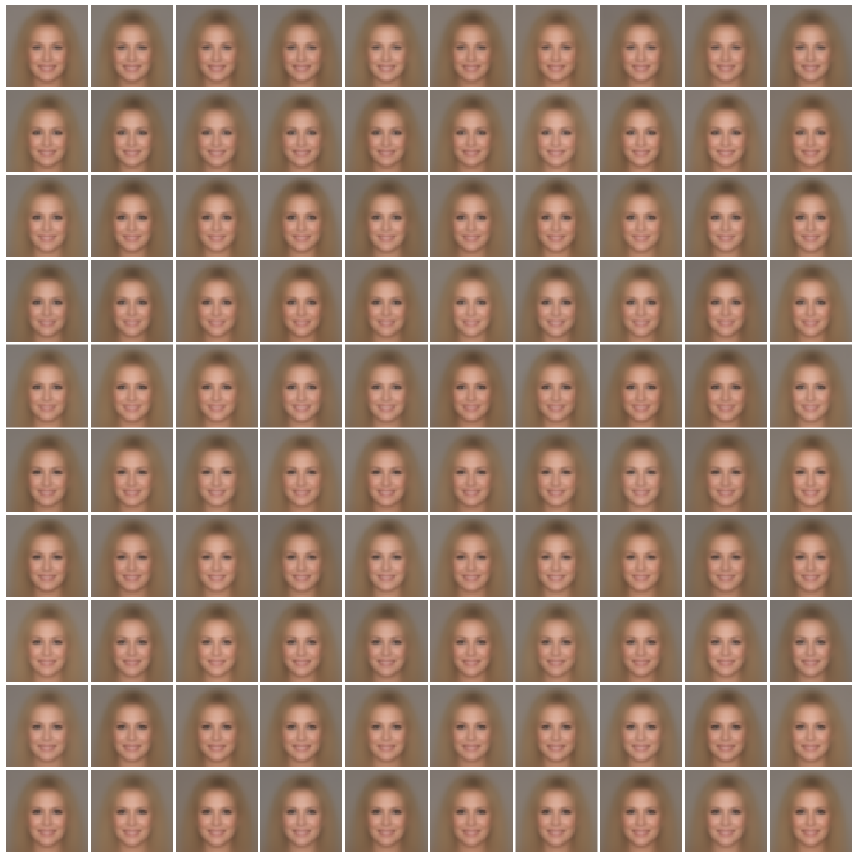}
\caption{SDS}
\label{fig:app_gen_sds}
\end{subfigure}
\begin{subfigure}{0.7\textwidth}
\centering
\includegraphics[width=\textwidth,height=0.4\textheight,keepaspectratio]{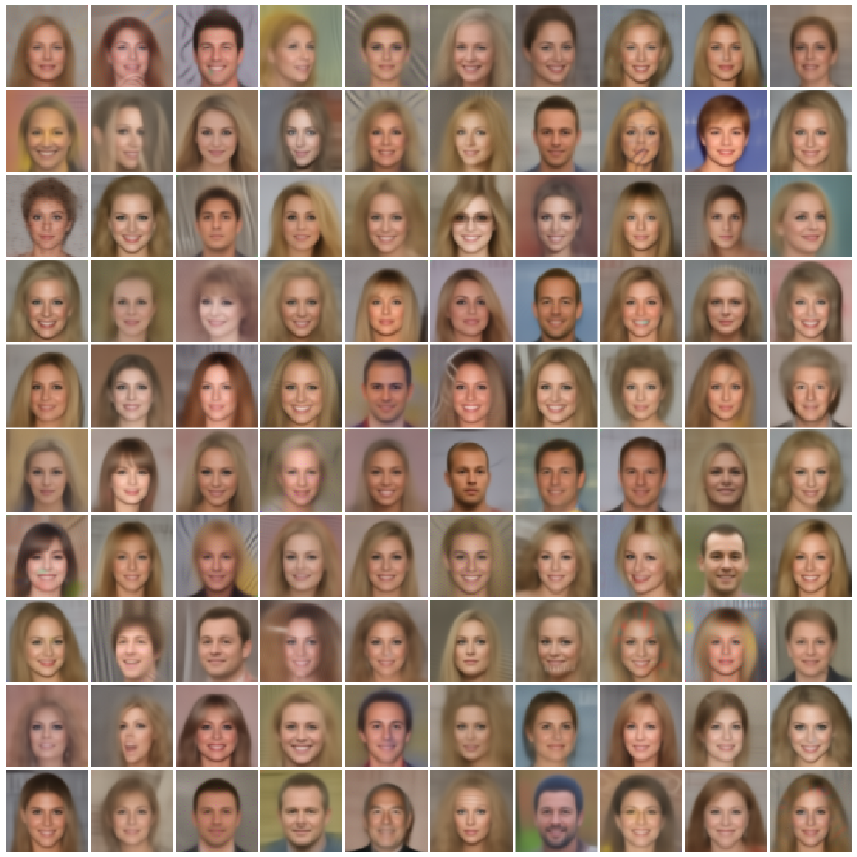}
\caption{SDI}
\label{fig:app_gen_sdi}
\end{subfigure}
\caption{Uncurated image generation results on CelebA, showing $100$ of the $128$ particles optimized per method.}
\label{fig:app_gen}
\end{figure}

\begin{figure}[H]
\ContinuedFloat
\centering
\begin{subfigure}{0.7\textwidth}
\centering
\includegraphics[width=\textwidth,height=0.4\textheight,keepaspectratio]{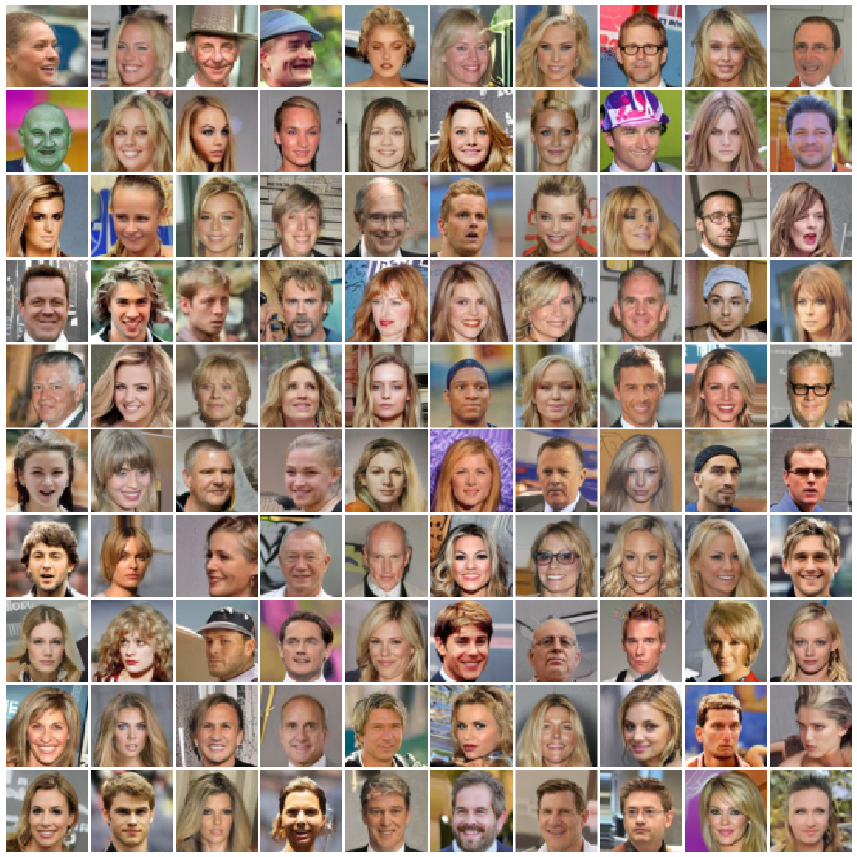}
\caption{VSD}
\label{fig:app_gen_vsd}
\end{subfigure}
\begin{subfigure}{0.7\textwidth}
\centering
\includegraphics[width=\textwidth,height=0.4\textheight,keepaspectratio]{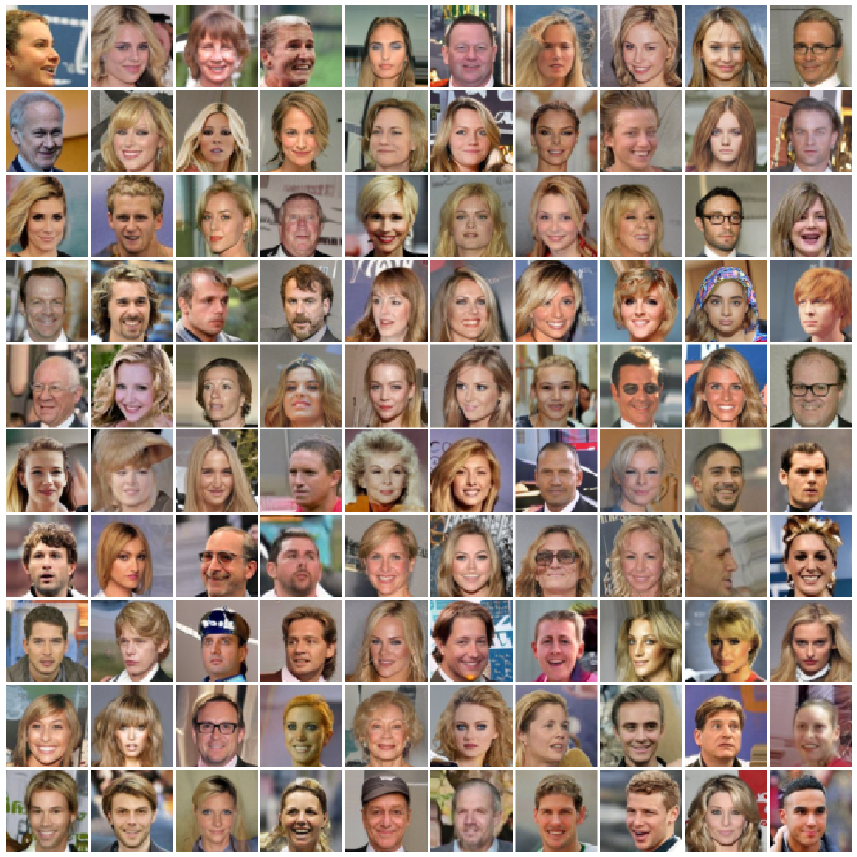}
\caption{PFD}
\label{fig:app_gen_pfd}
\end{subfigure}
\caption[]{Uncurated image generation results on CelebA showing $100$ of the $128$ particles optimized per method. (continued).}
\end{figure}

\FloatBarrier
\subsection{Symmetric Image Generation}
\label{app:symmetric}
\vspace{-1em}
\begin{figure}[H]
\centering
\includegraphics[width=0.67\linewidth]{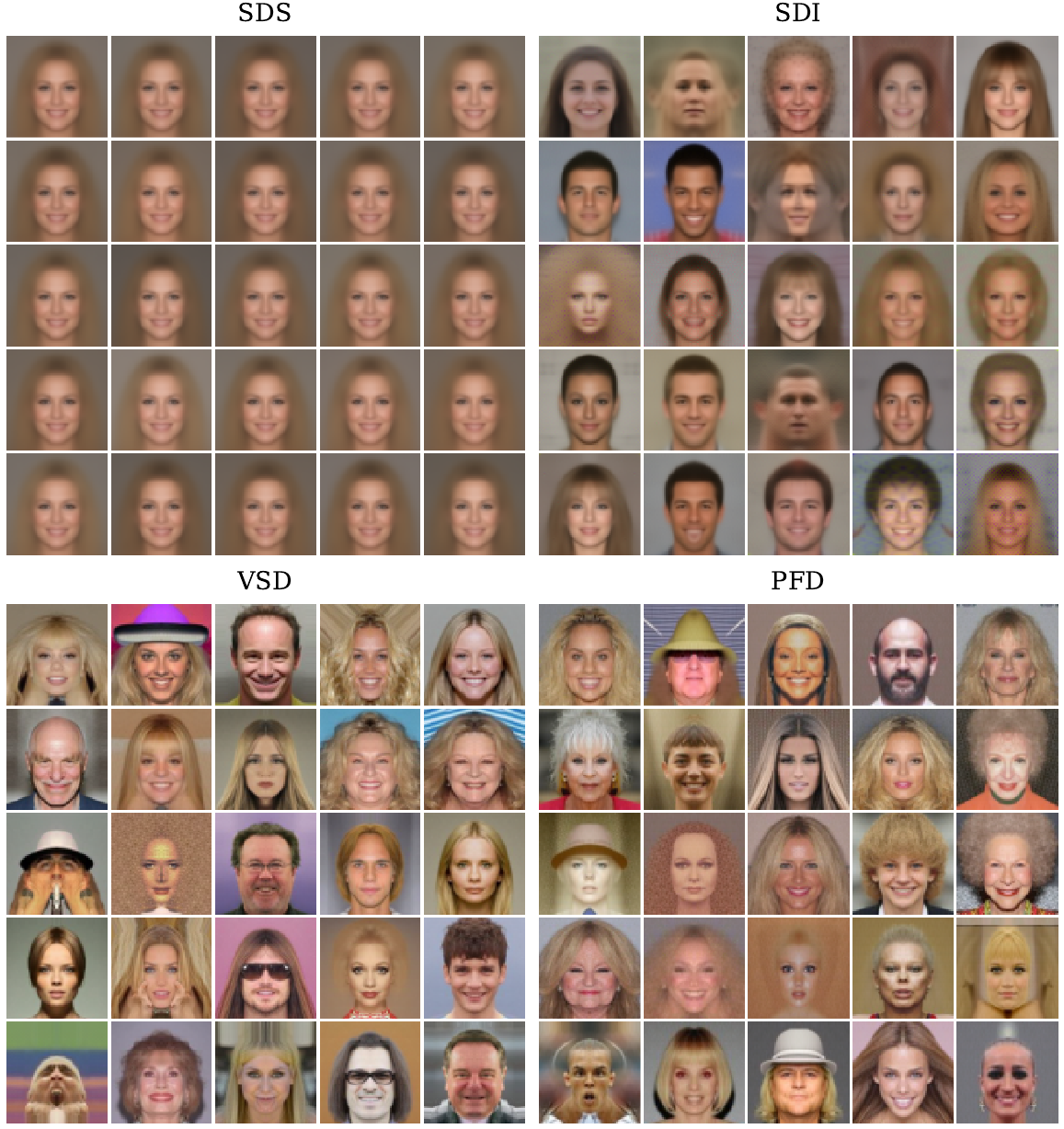}
\vspace{-0.8em}
\caption{Uncurated symmetric image generation results on CelebA, showing $25$ of the $128$ particles optimized per method.}
\label{fig:app_symmetric}
\end{figure}

\FloatBarrier
\vspace{-1em}
\subsection{Vector Graphics Generation}
\label{app:svg}
\vspace{-1em}
\begin{figure}[H]
\centering
\includegraphics[width=0.67\linewidth]{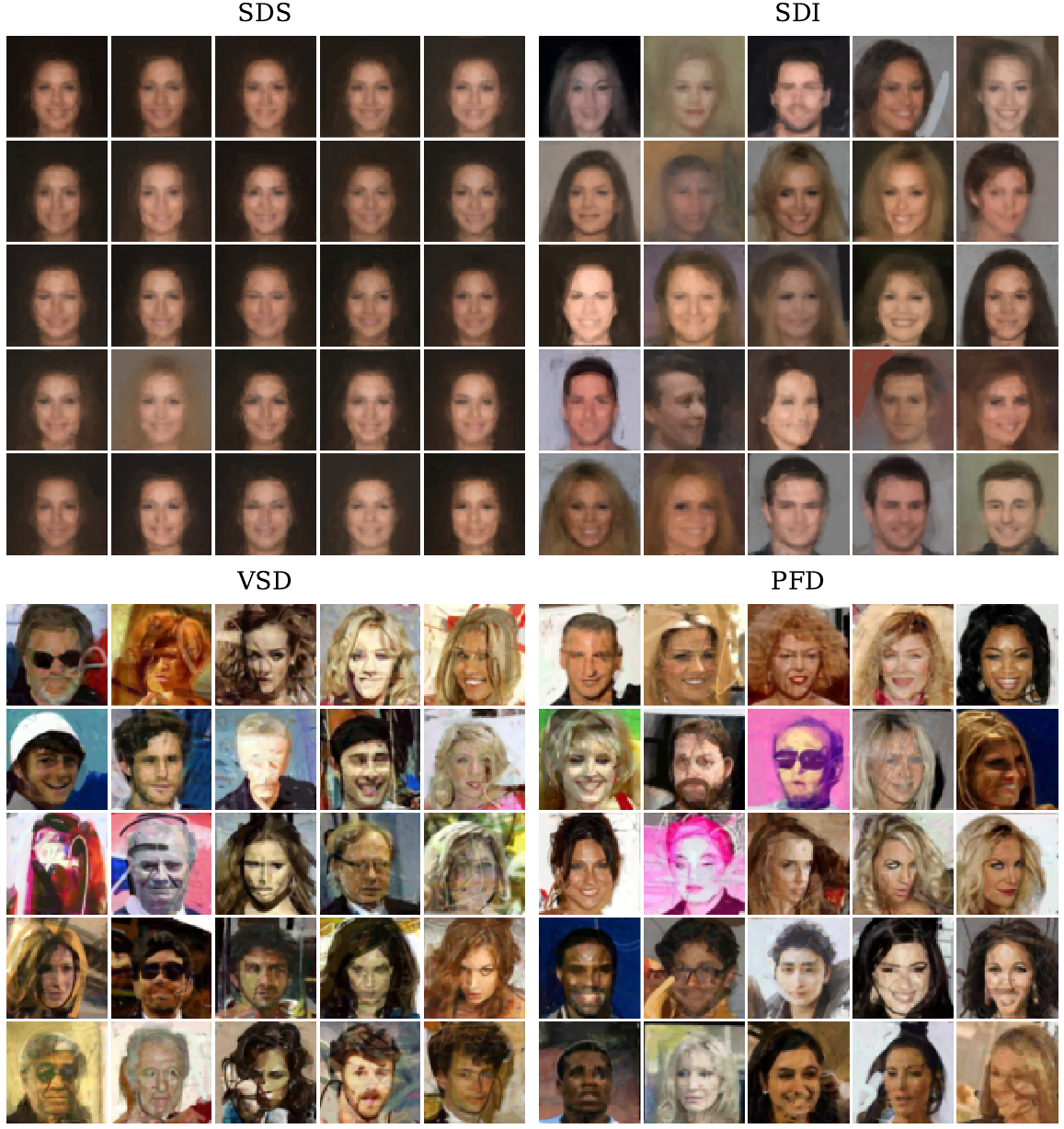}
\vspace{-0.8em}
\caption{Uncurated vector graphics generation results on CelebA, showing $25$ of the $128$ particles optimized per method.}
\label{fig:app_svg}
\end{figure}

\FloatBarrier
\subsection{Image Inpainting}
\label{app:inpainting}

\begin{figure}[H]
\centering
\begin{subfigure}{0.9\textwidth}
\centering
\includegraphics[width=\textwidth]{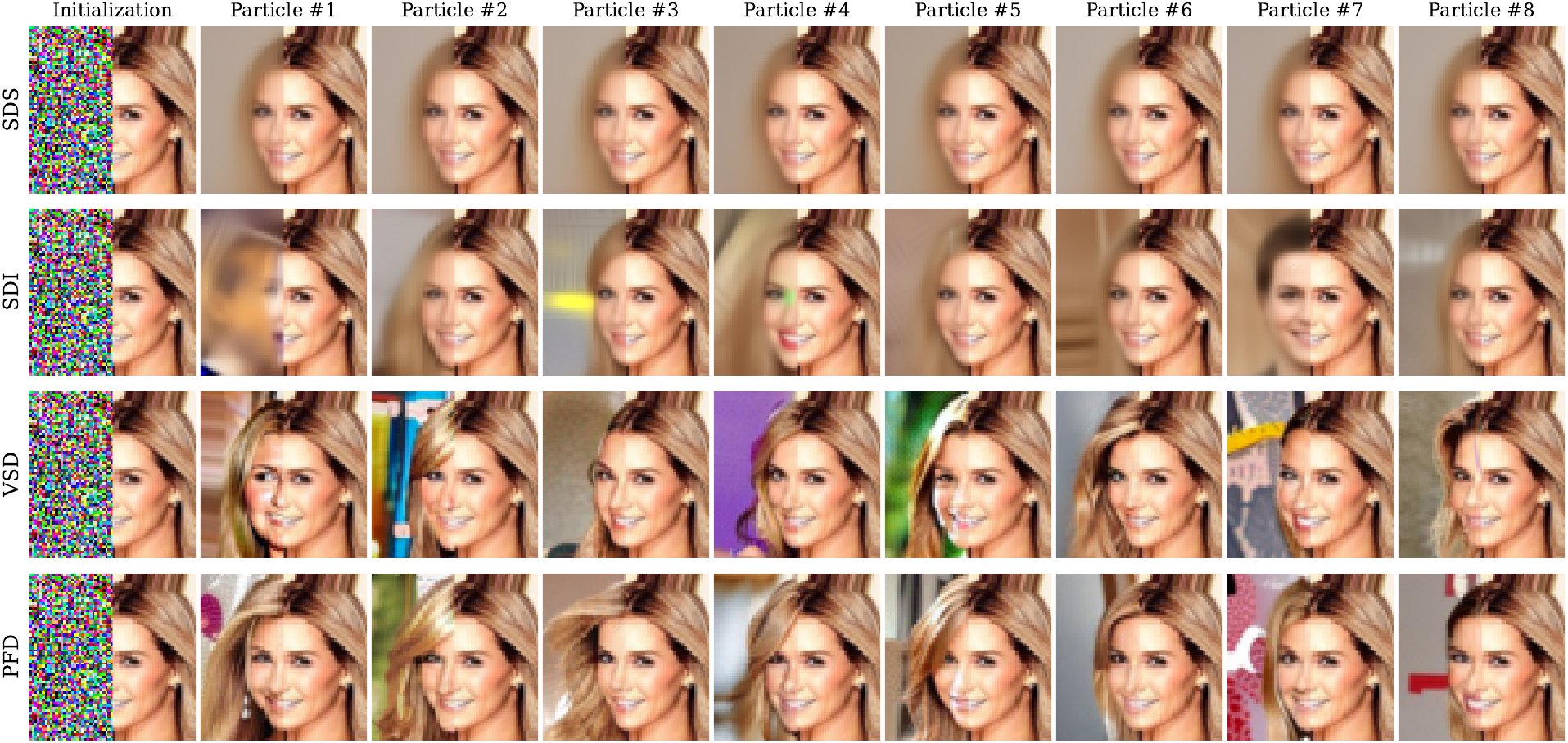}
\caption{Left half masked, right half observed.}
\label{fig:app_inpainting_a}
\end{subfigure}
\begin{subfigure}{0.9\textwidth}
\centering
\includegraphics[width=\textwidth]{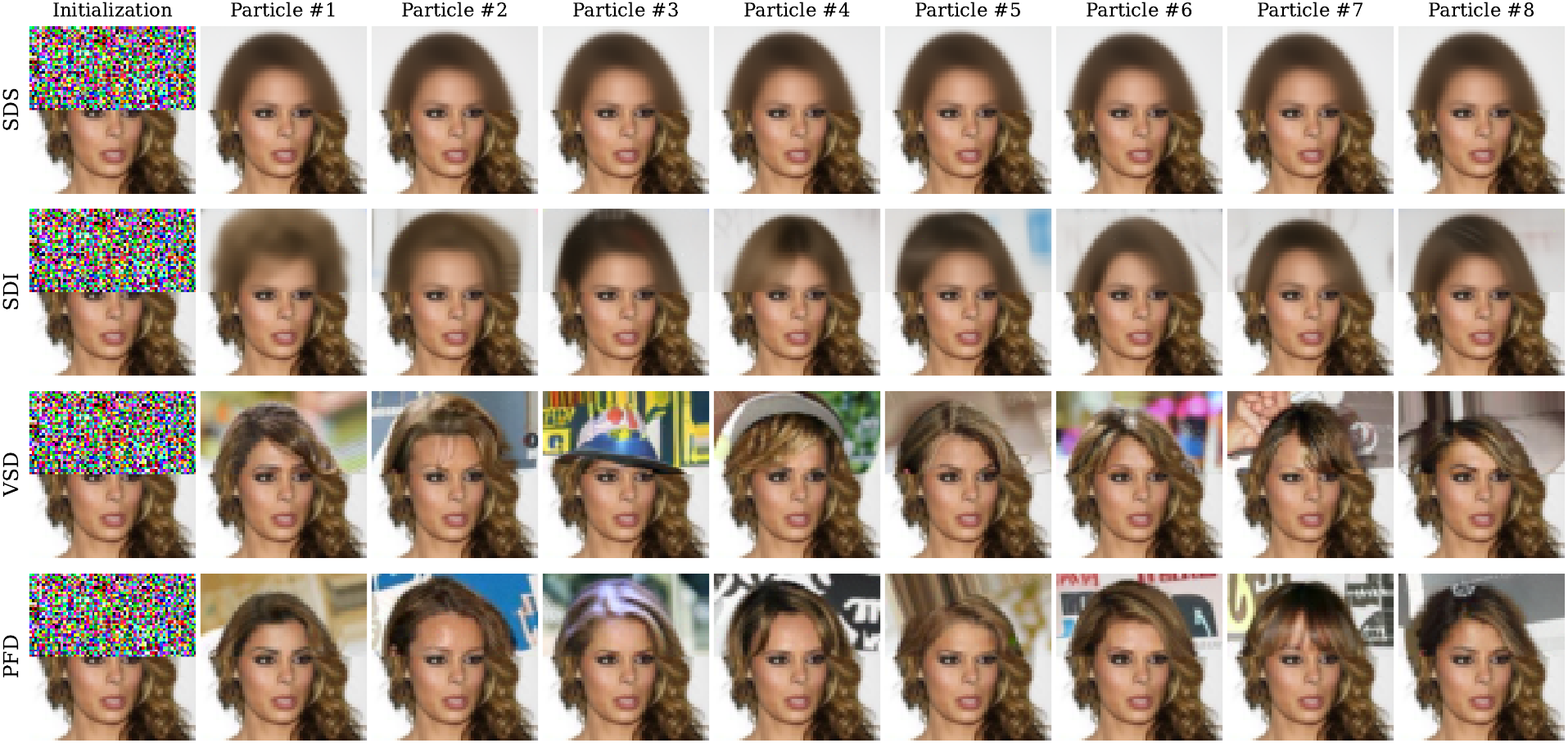}
\caption{Top half masked, bottom half observed.}
\label{fig:app_inpainting_b}
\end{subfigure}
\begin{subfigure}{0.9\textwidth}
\centering
\includegraphics[width=\textwidth]{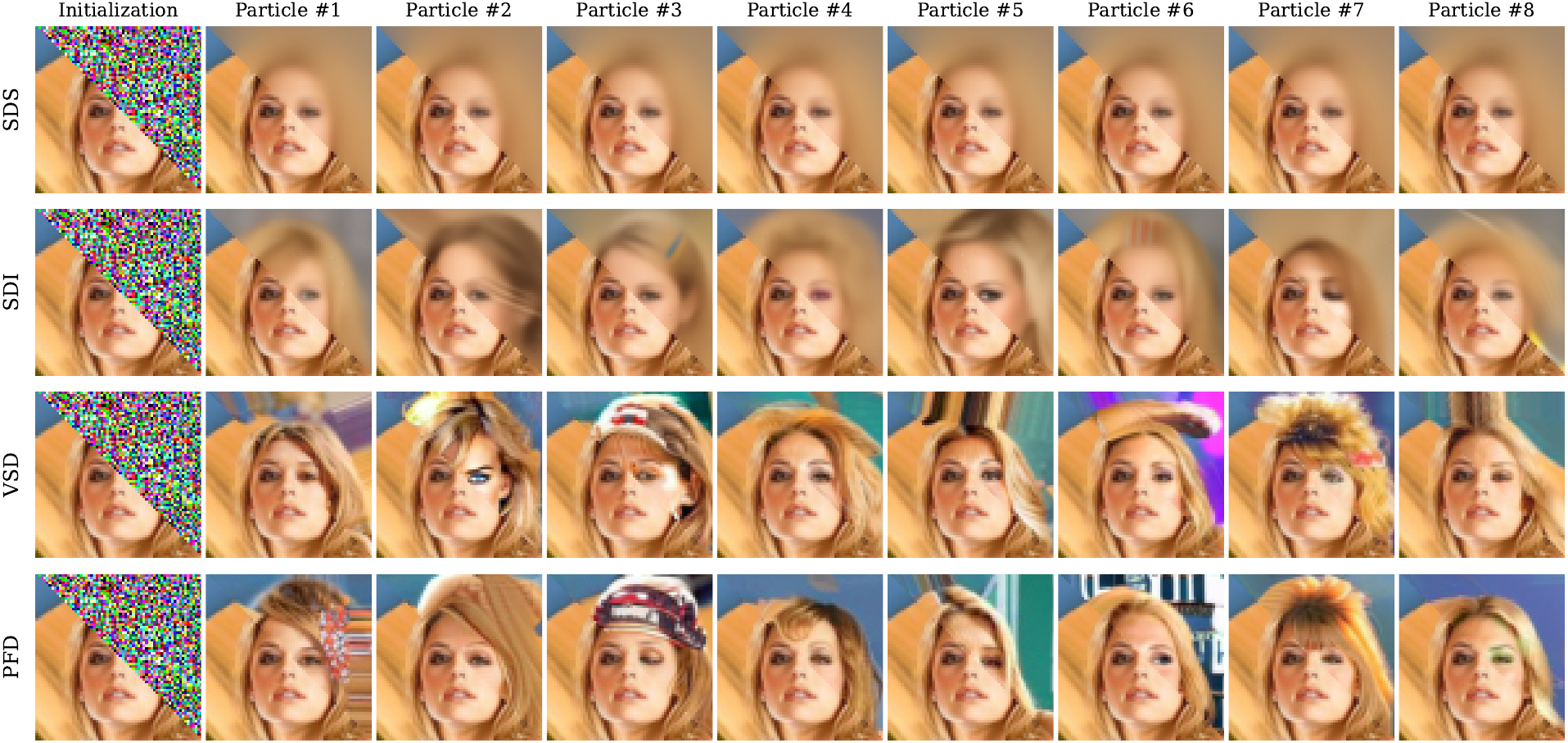}
\caption{Upper-right triangle masked, lower-left observed.}
\label{fig:app_inpainting_c}
\end{subfigure}
\caption{Image inpainting results on CelebA. In each panel, the first column shows the initialization, with noise in the masked region, and the remaining columns show eight random, uncurated particles of the $128$ optimized per method.}
\label{fig:app_inpainting}
\end{figure}

\FloatBarrier
\subsection{Text-to-3D Generation}
\label{app:text_to_3d}

\begin{figure}[H]
\centering
\includegraphics[width=0.8\linewidth]{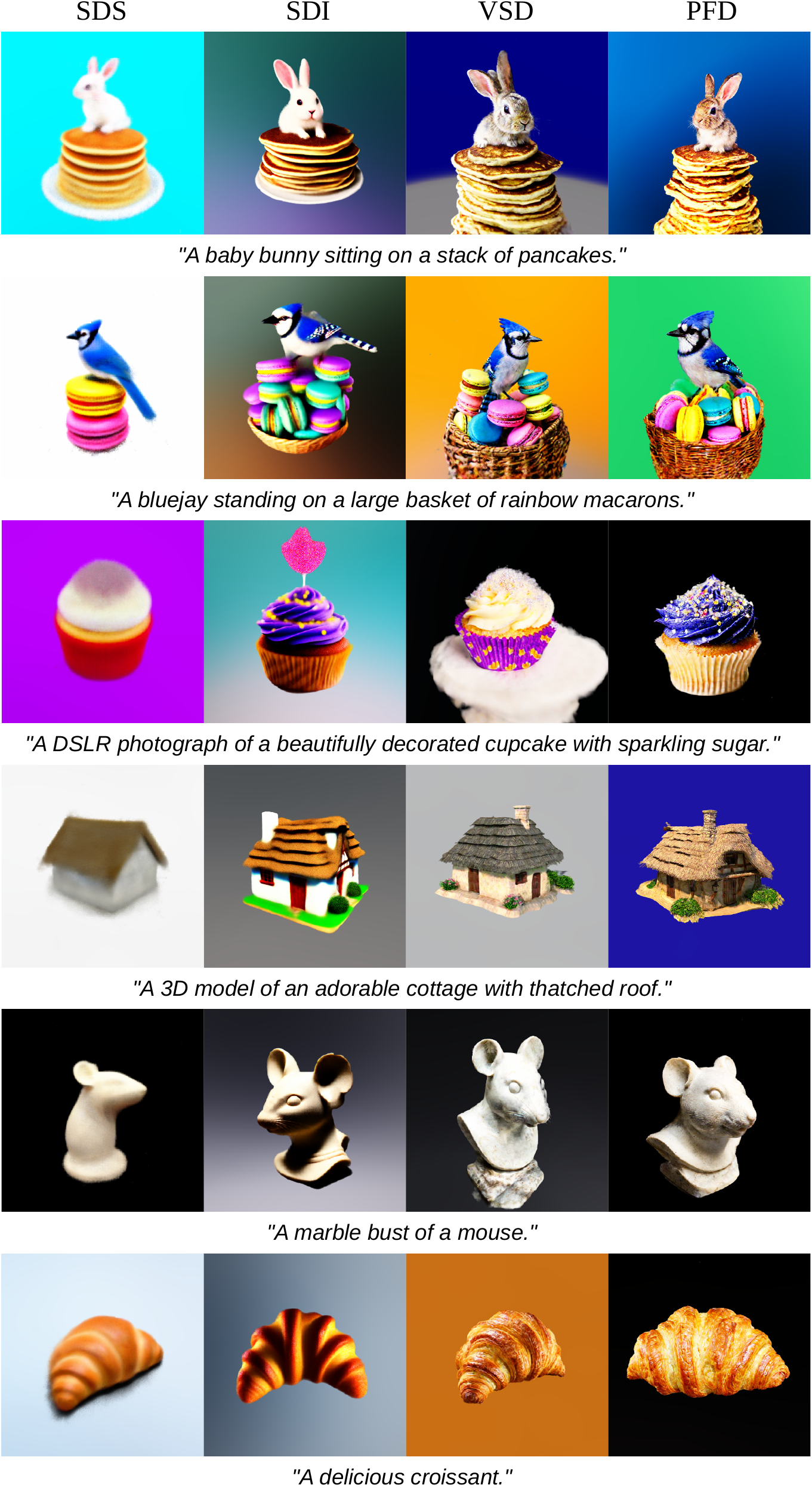}
\caption{Text-to-3D generation results with SDS, SDI, VSD, and PFD.}
\label{fig:app_t23d_1}
\end{figure}

\begin{figure}[H]
\centering
\includegraphics[width=0.8\linewidth]{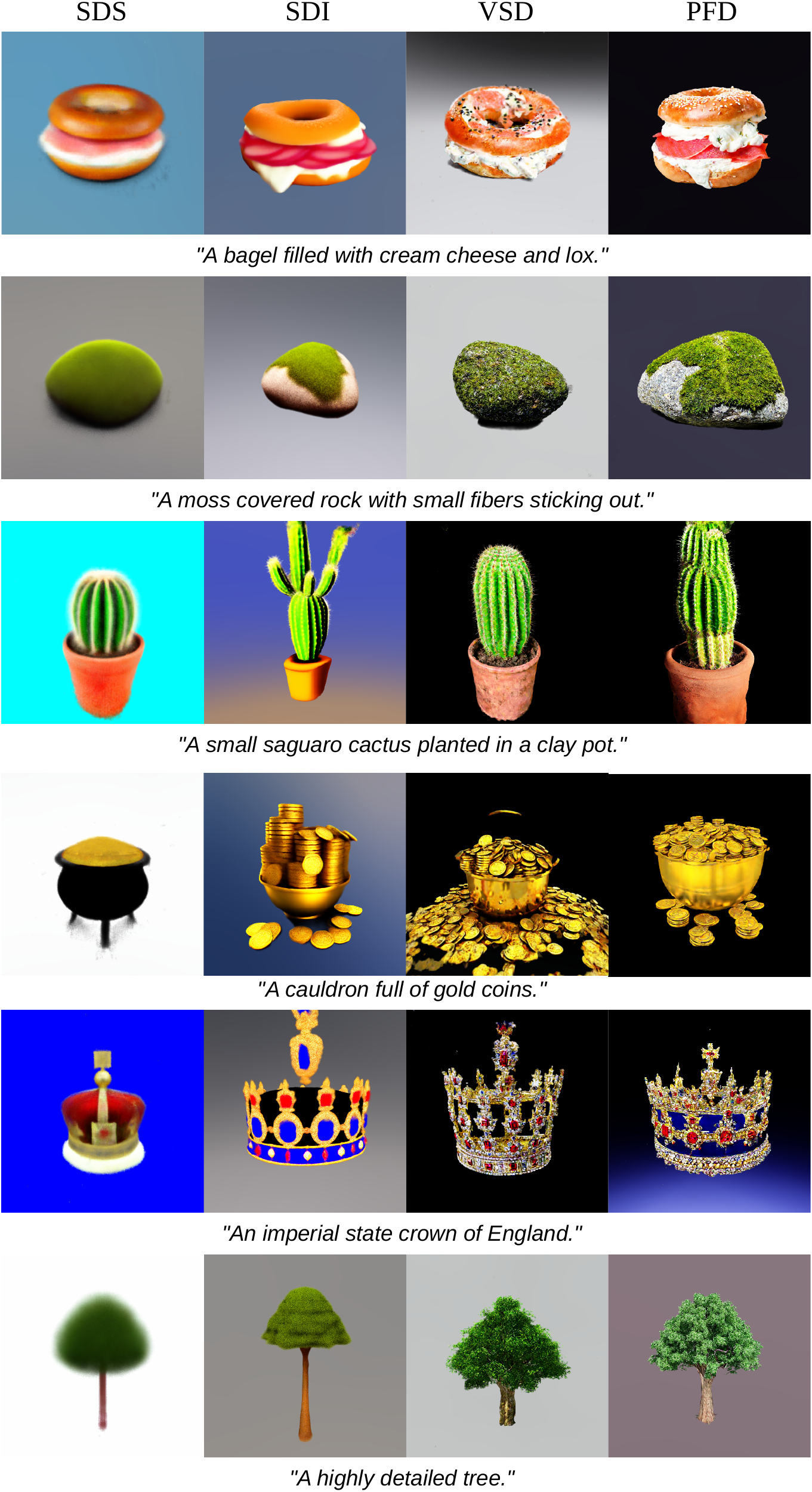}
\caption{Text-to-3D generation results with SDS, SDI, VSD, and PFD (continued).}
\label{fig:app_t23d_2}
\end{figure}

\begin{figure}[H]
\centering
\includegraphics[width=0.8\linewidth]{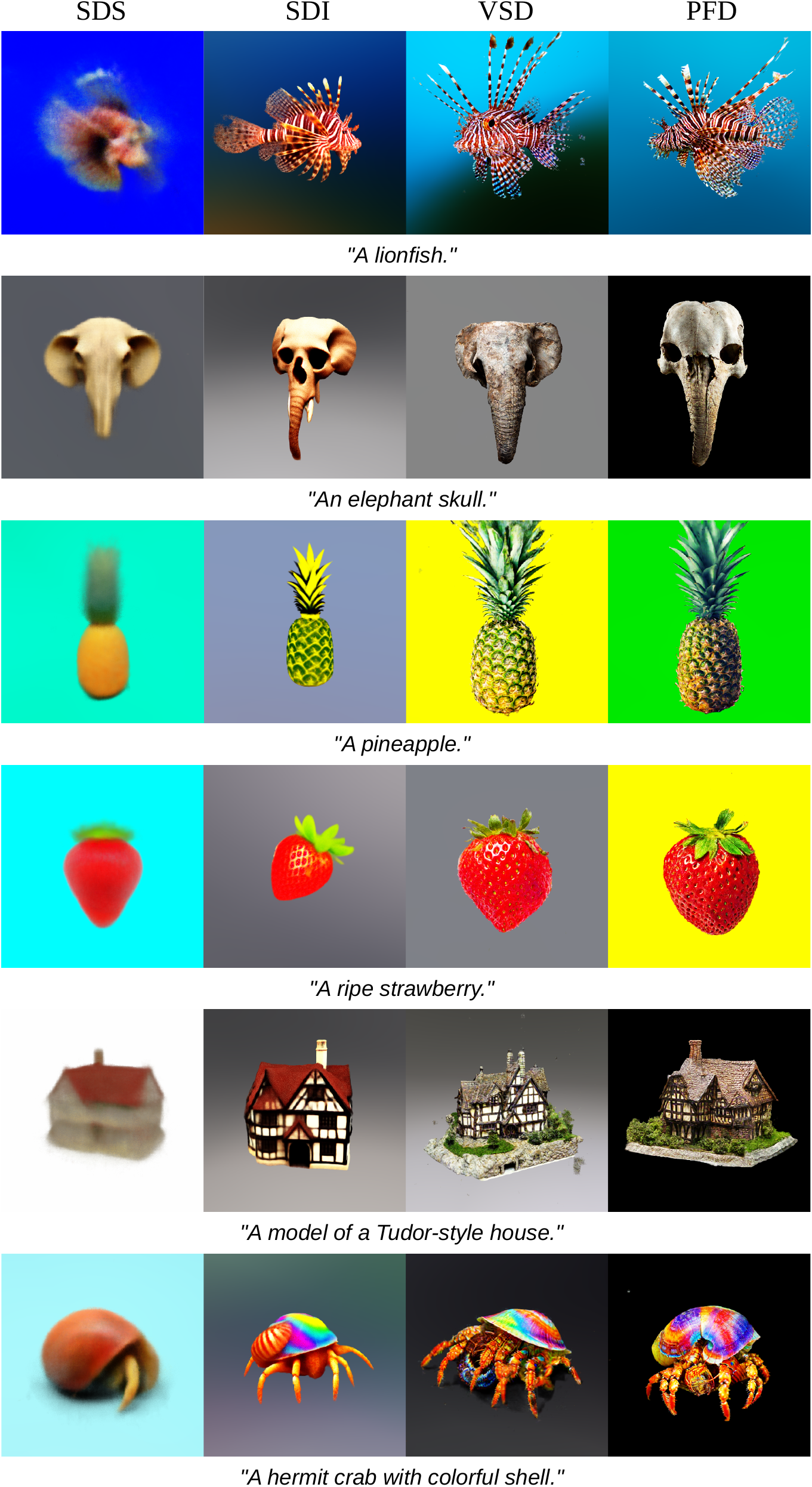}
\caption{Text-to-3D generation results with SDS, SDI, VSD, and PFD (continued).}
\label{fig:app_t23d_3}
\end{figure}
\FloatBarrier